\documentclass[sigconf, nonacm]{acmart}
\usepackage{graphicx}
\usepackage{bm}
\usepackage{amsmath}
\usepackage{amsfonts}
\usepackage{amsthm}
\usepackage{dsfont}
\usepackage{subcaption}
\usepackage{cleveref}
\usepackage{thmtools, thm-restate}
\usepackage[ruled,vlined,linesnumbered]{algorithm2e}
\usepackage[normalem]{ulem}
\usepackage{enumitem}
\usepackage{booktabs}

\definecolor{darkred}{RGB}{200,0,0}

\newtheorem{theorem}{Theorem}

\newcommand\numberthis{\addtocounter{equation}{1}\tag{\theequation}}

\newcommand{\algname}{\textsc{FSR}}
\newcommand{\algnamec}{\algname-\textsc{C}}
\newcommand{\algnameu}{\algname-\textsc{U}}
\newcommand{\algnameuub}{\algnameu-\textsc{UB}}

\newcommand{\pars}[1]{\left( #1 \right)}
\newcommand{\sqpars}[1]{\left[ #1 \right]}
\newcommand{\sqparsmid}[1]{\Bigl[ #1 \Bigr]}
\newcommand{\brpars}[1]{\left\lbrace #1 \right\rbrace}
\newcommand{\qt}[1]{\lq\lq#1\rq\rq}
\newcommand{\qtm}[1]{\text{\lq\lq}#1\text{\rq\rq}}
\newcommand{\ind}[1]{\mathds{1} \left[ #1 \right] }

\newcommand{\dataset}{\mathcal{D}}
\newcommand{\sample}{\mathcal{S}}
\newcommand{\A}{\mathcal{A}}

\newcommand{\lang}{\mathcal{L}}
\newcommand{\patt}{\mathcal{P}}
\newcommand{\tqual}{ \mathsf{q}_{\patt}}
\newcommand{\qualp}{{\qual{\patt}{\dataset}}}
\newcommand{\nqualp}[2]{{\bar{\mathsf{q}}^{*}_{#1}(#2)}}
\newcommand{\squalp}{\squal{\patt}{\dataset}}
\newcommand{\squal}[2]{ \mathsf{q}_{#1}(#2) }
\newcommand{\qual}[2]{ \bar{\mathsf{q}}_{#1}(#2) }
\newcommand{\hsqual}[3]{\bar{\mathsf{q}}_{#1}(#2,#3) }
\newcommand{\tfreq}{ \mathsf{f}_{\patt}}
\newcommand{\freqp}{\mathsf{f}_{\patt}(\dataset)}
\newcommand{\unusp}{\mathsf{u}_{\patt}(\dataset)}
\newcommand{\covp}{\mathsf{C}_{\patt}(\dataset)}

\newcommand{\devest}{\tilde{d}(\resset)}
\newcommand{\devesth}{\tilde{d}(\resseth)}

\newcommand{\E}{\mathop{\mathbb{E}}}
\newcommand{\F}{\mathcal{F}}
\newcommand{\X}{\mathcal{X}}
\newcommand{\T}{\mathcal{T}}

\newcommand{\abs}[1]{\lvert#1 \rvert}

\newcommand{\probdist}{\gamma}

\newcommand{\R}{\mathbb{R}}

\newcommand{\resset}{\mathcal{R}^\star}
\newcommand{\resseth}{\hat{\mathcal{R}}^\star}

\newcommand\given[1][]{\:#1\vert\:}

\newcommand{\unionbdelta}{\ln \bigl( \frac{4}{\delta} \bigr)}

\newif\ifextversion
\extversiontrue

\newcommand\vldbdoi{XX.XX/XXX.XX}
\newcommand\vldbpages{XXX-XXX}
\newcommand\vldbvolume{17}
\newcommand\vldbissue{10}
\newcommand\vldbyear{2024}
\newcommand\vldbauthors{\authors}
\newcommand\vldbtitle{\shorttitle} 
\newcommand\vldbavailabilityurl{https://github.com/VandinLab/FSR}
\newcommand\vldbpagestyle{plain} 

\begin{document}\title[Efficient Discovery of Significant Patterns with Few-Shot Resampling]{Efficient Discovery of Significant Patterns \\ with Few-Shot Resampling}

\author{Leonardo Pellegrina}
\affiliation{%
  \institution{Dept. of Information Engineering, University of Padova}
  \streetaddress{Via Gradenigo 6b}
  \city{Padova}
  \state{Italy}
  \postcode{35129}
}
\email{leonardo.pellegrina@unipd.it}

\author{Fabio Vandin}
\affiliation{%
  \institution{Dept. of Information Engineering, University of Padova}
  \streetaddress{Via Gradenigo 6b}
  \city{Padova}
  \state{Italy}
  \postcode{35129}
}
\email{fabio.vandin@unipd.it}

\begin{abstract}
Significant pattern mining is a fundamental task in mining transactional data, 
requiring to identify \emph{patterns} significantly associated with the value of a given feature, the \emph{target}. 
In several applications, such as biomedicine, basket market analysis, and social networks, the goal is to discover patterns whose association with the target is defined with respect to an underlying population, or process, of which the dataset represents only a collection of observations, or samples. A natural way to capture the association of a pattern with the target is to consider its \emph{statistical significance}, assessing its deviation from the (null) hypothesis of independence between the pattern and the target. While several algorithms have been proposed to find statistically significant patterns, it remains a computationally demanding task, and for complex patterns such as subgroups, no efficient solution exists.

We present \algname,  an efficient algorithm to identify statistically significant patterns with rigorous guarantees on the probability of false discoveries. \algname\ builds on a novel general framework for mining significant patterns that captures some of the most commonly considered patterns, including itemsets, sequential patterns, and subgroups. \algname\ uses a small number of resampled datasets, obtained by assigning i.i.d. labels to each transaction, to rigorously bound the supremum deviation of a quality statistic measuring the significance of patterns. \algname\ builds on novel tight bounds on the supremum deviation that require to mine a small number of resampled datasets, while providing a high effectiveness in discovering significant patterns. As a test case, we consider significant subgroup mining, and our evaluation on several real datasets shows that \algname\ is effective in discovering significant subgroups, while requiring a small number of resampled datasets.
\end{abstract}

\maketitle

\pagestyle{\vldbpagestyle}
\begingroup\small\noindent\raggedright\textbf{PVLDB Reference Format:}\\
\vldbauthors. \vldbtitle. PVLDB, \vldbvolume(\vldbissue): \vldbpages, \vldbyear.\\
\href{https://doi.org/\vldbdoi}{doi:\vldbdoi}
\endgroup
\begingroup
\renewcommand\thefootnote{}\footnote{\noindent
This work is licensed under the Creative Commons BY-NC-ND 4.0 International License. Visit \url{https://creativecommons.org/licenses/by-nc-nd/4.0/} to view a copy of this license. For any use beyond those covered by this license, obtain permission by emailing \href{mailto:info@vldb.org}{info@vldb.org}. Copyright is held by the owner/author(s). Publication rights licensed to the VLDB Endowment. \\
\raggedright Proceedings of the VLDB Endowment, Vol. \vldbvolume, No. \vldbissue\ %
ISSN 2150-8097. \\
\href{https://doi.org/\vldbdoi}{doi:\vldbdoi} \\
}\addtocounter{footnote}{-1}\endgroup

\ifdefempty{\vldbavailabilityurl}{}{
\vspace{.3cm}
\begingroup\small\noindent\raggedright\textbf{PVLDB Artifact Availability:}\\
The source code, data, and/or other artifacts have been made available at \url{https://github.com/VandinLab/FSR} 
\endgroup
}


\section{Introduction}
\label{sec:intro}

\emph{Pattern mining} is a fundamental task in data mining that, in its most common definition~\cite{han2007frequent}, requires to find patterns that occur more often than a given frequency threshold in a database of transactions. Pattern mining finds applications in several areas such as market basket analysis~\cite{agrawal1993mining}, graph databases~\cite{aggarwal2010dense,al2009output,chen2009mining}, and the analysis of spatial and temporal data~\cite{cao2019efficient,ceccarello2022fast,ho2022efficient}.

\emph{Significant pattern mining}~\cite{PellegrinaRV19b,hamalainen2019tutorial} is an extension of pattern mining that, in its most general formulation, requires to discover \emph{patterns} with a \emph{significant association} with a binary label from  a dataset consisting of a collection of elements, where each element comprises the values of \emph{features}, which may be categorical, binary, or continuous, and the value of the binary label of interests, also called the \emph{target}. Such formulation captures various types of patterns, such as \emph{itemsets}, when all features are binary, or \emph{subgroups}~\cite{atzmueller2015subgroup}, with more general features. This task finds applications in a wide range of domains, such as market basket analysis, medicine, and molecular biology, where finding reliable associations is paramount.

Significance is usually assessed using the \emph{statistical hypothesis testing} framework. In such framework one defines a measure of \emph{quality} for patterns, and assumes the \emph{null hypothesis} of no association between a pattern and the target label. The \emph{significant patterns} are then the ones with quality that significantly deviate from the \emph{null distribution}, that is, the distribution of the quality under the null hypothesis. The deviation from the null distribution is usually measured by a $p$-value, that is the probability, under the null distribution, that the pattern has quality as large as the one observed in the dataset.

A major complication in the use of the statistical hypothesis testing framework in data mining is given by the huge number of candidate patterns that are considered, resulting in a \emph{multiple hypothesis testing} problem. With a huge number of candidate patterns, some non-significant patterns display a substantial deviation from the null distribution just by chance. Therefore, it is critical to account for testing multiple hypotheses when mining significant patterns, in order to avoid reporting a large number of spurious discoveries.  Several methods have been proposed to deal with multiple hypothesis testing~\cite{benjamini1995controlling,bonferroni1936teoria,westfall1993resampling}. While these methods provide various guarantees, the one most commonly considered is the Family-Wise Error Rate (FWER), which is the probability of reporting in output one or more false discoveries.

Current approaches for significant pattern mining with guarantees on the FWER belong to one of two classes. The first class is given by approaches that assess the significance of each single pattern (e.g., through a $p$-value or related quantities), and then perform an analytical correction to account for multiple hypothesis testing~\cite{webb2006discovering,webb2007discovering,webb2008layered}. A widely used procedure in this class is given by Bonferroni correction~\cite{benjamini1995controlling}, which computes a \emph{corrected $p$-value} by multiplying the $p$-value $p_{\patt}$ of a pattern $\patt$ by the number $h$ of candidate hypotheses. If patterns with corrected $p$-value $h \times p_{\patt}$ below a threshold $\alpha$ are flagged as significant, the FWER of the output is guaranteed to be $\le \alpha$. While these approaches and their improved versions~\cite{TeradaOHTS13,minato2014fast} are fairly efficient, thanks to the use of analytical derivations, they suffer from a low \emph{statistical power}, that is they often fail in identifying significant patterns, due to the multiple hypothesis corrections that must provide guarantees for every situation, independently of the observed data.

The second class is given by approaches that use the dataset to estimate the overall distribution of the patterns' statistics (or corresponding measures, such as the $p$-values) under the null hypothesis~\cite{llinares2015fast,pellegrina2020efficient,terada2015high}. The distribution is estimated using permuted versions of the data, obtained by keeping the features in each element of the dataset fixed and randomly permuting the target labels among elements. For example, the Westfall-Young (WY) method~\cite{westfall1993resampling} uses permuted datasets to estimate the quantiles of the \emph{smallest} $p$-value under the null hypotheses (or, equivalently, largest qualities or test statistics), and uses such estimate to derive a corrected threshold to flag patterns as significant. 
These approaches usually improve the statistical power for detecting significant patterns compared to approaches in the first class, but are often computationally demanding, since they need to mine a large number of permuted datasets to obtain good estimates of the overall distribution of the patterns' statistics. While this can be achieved fairly efficiently for simple patterns such as itemsets~\cite{llinares2015fast,pellegrina2020efficient}, the overall approach is impractical for more complex patterns such as subgroups, for which mining even a single dataset is extremely time-consuming.

An additional limitation of permutational approaches is that they focus on \emph{conditional testing}~\cite{fisher1922interpretation}. In conditional testing one assumes that the variables of interest, in our case the frequency of patterns and the fraction of elements with target label $1$, are the same in \emph{every} dataset from the null distribution. In contrast, in \emph{unconditional testing}~\cite{barnard1945new} one assumes that the variables of interest are the realization of corresponding random variables. Conditional testing and unconditional testing capture different assumptions regarding how data is generated and collected, that is, whether the variables of interest would be the same in different repetitions of the experiment. The choice between the two types of testing depends on the specific scenarios. However, in practice conditional tests are often used for computational reasons, since unconditional tests are much more demanding from the computational standpoint due to the need to account for uncertainties in the observed quantities. In fact, while for simple patterns such as itemsets~\cite{PellegrinaRV19a} significant pattern mining procedures with (partial) unconditional testing have been designed, for more complex patterns such as subgroups no unconditional testing procedure is available.

\subsection{Contributions}
\label{sec:contribs}
This work focuses on the efficient discovery of significant patterns. Our contributions are four-fold. Firstly, 
we propose the first general framework to discover significant patterns that can be used for both conditional and unconditional testing. Our framework is based on a natural definition of a pattern's quality that captures its significance, and applies to any type of pattern for which the \emph{appearance} of the pattern in an element of the dataset is well defined. Such patterns include widely used patterns such as itemsets, subgroups, sequential patterns, and subgraphs.  Second, we propose \algname, an algorithm for the efficient discovery of a rigorous approximation of significant patterns while controlling the \emph{Family-Wise Error Rate}, which is the probability of reporting even a single false discovery. \algname\ uses a \emph{few-shot resampling approach}, that is, it mines a small number of \emph{resampled} datasets, obtained by keeping the features of each element of the dataset fixed and assigning i.i.d. values to the target. Moreover, \algname\ can leverage any existing algorithm for mining the patterns of interest. Third, we provide novel tight theoretical results relating the distribution of patterns' qualities under the conditional and the unconditional distributions, and relating the estimated maximum deviation of patterns' qualities in resampled datasets with their (unknown) true quality in the corresponding distribution. These results are crucial in making our approach practical, since they imply that mining a small number of resampled datasets is enough to identify significant patterns. Fourth, we consider significant subgroups mining as a test case in our extensive empirical evaluation. We use our algorithm \algname\ to derive the first approach for mining significant subgroups with unconditional testing, and an approach for the conditional testing scenario which is much more efficient than permutation testing while maintaining an extremely high power. For the most challenging datasets, \algname\ is the only approach that allows to mine significant subgroups within reasonable time.
More importantly, we remark that the considered test case of subgroups is well representative of other settings. 
In fact,  we expect the sensible improvements obtained by \algname\ to transfer to other cases (i.e., other pattern types), given the generality of our framework and the characteristics of our permutational approaches, which are shared by all significant pattern tasks. 

Due to space constraints, we defer some of the proofs and additional results to the 
\ifextversion
Appendix. 
\else
Appendix in the online extended version~\cite{fsrextended}. 
\fi

\section{Related Works}
\label{sec:relworks}

Our work focuses on efficiently mining significant patterns.  We now discuss the previous works most related to our contributions. We refer the reader to recent comprehensive reviews and tutorials for an overview of commonly used techniques for mining significant patterns~\cite{hamalainen2019tutorial,PellegrinaRV19b}.

Several approaches~\cite{webb2006discovering,webb2007discovering,webb2008layered} have used general methods for multiple hypotheses testing, such as Bonferroni~\cite{bonferroni1936teoria} and Holm methods~\cite{holm1979simple}, within significant pattern mining. Such methods result in low \emph{statistical power}, since they correct the $p$-value, or a measure of the significance of a pattern, by the number of candidate hypotheses (i.e., the number of patterns), which is extremely large in data mining applications. LAMP~\cite{TeradaOHTS13,minato2014fast} is a recently introduced method that partially addresses such issue by  selecting a subset of patterns for testing, while discarding the patterns with no chance of being significant. Such approach leads to identifying an improved (i.e., smaller) correction factor, resulting in higher statistical power while controlling the FWER. 
However, since each selected pattern is still tested as a distinct hypothesis, LAMP still leads to reduced statistical power \cite{llinares2015fast,terada2015high}. 
Our algorithm \algname\ instead uses resampled datasets, taking into account dependencies among patterns, to achieve high statistical power.

Several permutation-based methods have been proposed to identify patterns significantly associated with a binary target label.  \cite{llinares2015fast,pellegrina2020efficient,terada2015high} use the Westfall-Young (WY) permutation test~\cite{westfall1993resampling}. The WY permutation procedure requires to estimate the $\delta$-quantile of the distribution of the minimum (overall all patterns) $p$-value  (or, equivalently, of the distribution of the maximum deviation, over all patterns, of the measure of significance). The use of such quantiles makes WY permutation testing more powerful than LAMP, but also more computationally expensive, since the estimation of such quantiles requires to mine a large number of permuted datasets. For these reasons, all such methods may need impractical resources.
 Our algorithm \algname\ instead requires to mine a small number of resampled datasets, leading to an efficient approach even for complex patterns such as subgroups. In addition, permutation-based methods only consider the conditional distribution, where patterns' frequencies and the fraction of elements with target label equal to $1$ are assumed fixed in all permuted datasets. Our approach instead works for both the conditional and the unconditional distribution.

The identification of significant patterns with unconditional testing has received scant attention. This is due, in part, to the higher computational cost required for assessing the significance of even a single hypothesis in the unconditional setting, for example using Barnard's test~\cite{barnard1945new}. As far as we know, the only approach for mining significant patterns in a \emph{partially} unconditional setting is \texttt{SPuManTE}~\cite{PellegrinaRV19a}. In the partially unconditional setting considered by \texttt{SPuManTE} only the frequencies of the patterns are not fixed, while the target is fixed by design. In contrast our algorithm \algname\ considers a \emph{fully} unconditional setting, where the fraction of elements with target label equal to $1$ is not fixed either. Moreover, \texttt{SPuManTE} leverages specialized techniques for mining significant itemsets that cannot be easily generalized to other pattern types, while our approach applies directly to several pattern languages, including itemsets, sequential patterns, and subgroups.

Other works, orthogonal to ours, consider improving the diversity or  limiting redundancy of the output~\cite{van2012diverse,kalofolias2017efficiently,dalleiger2022discovering}.

\section{Preliminaries}
\label{sec:prelims}
We consider a \emph{dataset}  $\dataset$ as a collection of $m$ transactions $\dataset = \brpars{ (s_1 , \ell_1) , \dots , (s_m , \ell_m) }$, where 
each transaction $(s , \ell) \in \dataset$ is composed by a set $s$ of $d$ \emph{features}, either binary, categorical, or continuous, and a binary \emph{target} variable $\ell \in \{0 , 1\}$. 
More generally, we assume that $s$ belongs to a domain $\mathcal{X}$. 
By defining the multisets  $\A = \brpars{ s_1 , \dots , s_m }$ and $\T = \brpars{ \ell_1 , \dots , \ell_m }$, a dataset $\dataset$ is also represented by the pair $\dataset = (\A , \T)$. 
We assume to have a \emph{language} $\lang$ containing the \emph{patterns} of potential interest. This scenario captures widely used pattern mining tasks, such as: itemset mining, where all features correspond to (binary) \emph{items} and the language $\lang$ corresponds to the set of all \emph{itemsets}, i.e., (non-empty) subsets of items; subgroup mining, where the language $\lang$ contains all \emph{subgroups}, i.e., the sets of conjunctions with at most $z$ conditions over features from $\A$, where each condition is either an equality, on a categorical feature, or an interval, on a continuous feature.

Given a transaction $(s,\ell)$, we use the notation $\patt \in s$ to say that the set $s$ of features  \emph{supports} pattern $\patt$, where the meaning of $\patt \in s$ depends on the specific data mining task. For example, for itemset mining, $\patt \in s$ means that the pattern $\patt$ is contained in the set $s$, while for subgroup mining it means that the conditions defined by $\patt$ are all satisfied by the features of $s$. We define the set $\covp$ of transactions in the dataset $\dataset$ that \emph{support} a pattern $\patt$  as
$\covp = \brpars{ (s , \ell) \in \dataset : \patt \in s }$. The \emph{frequency} $\freqp$ of a pattern $\patt$ in the dataset $\dataset$ is the fraction of transactions of $\dataset$ that support $\patt$: 
$\freqp = \frac{1}{m} \sum_{i=1}^m \ind{ \patt \in s_i } = \frac{|\covp|}{m} $.

Finally,  let $\mu(\dataset)$  denote the average value of the target $\ell$ for transactions in the dataset $\dataset$:
$\mu(\dataset) = \frac{1}{|\dataset|} \sum_{(s , \ell) \in \dataset} \ell $.

\subsection{Significant Patterns}
\label{sec:sig_sub}

Our goal is to find \emph{significant patterns}, where a pattern $\patt$ is \emph{significant} if the presence of $\patt$ in a transaction is associated with the target variable of the transaction being $1$. In particular, we consider the significance of a pattern in the framework of \emph{statistical} significance, assuming that the transactions $\brpars{ (s_1 , \ell_1) , \dots , (s_m , \ell_m) }$ constituting the dataset $\dataset$ are \emph{samples} from an unknown distribution $\probdist$. A pattern $\patt$ is associated with the target variable if the probability of the event ``$\patt \in s$ \emph{and} $\ell=1$'' is higher than the corresponding probability when the event ``$\patt \in s$'' and the event ``$\ell=1$'' are independent. Formally, this corresponds to consider the following \emph{null hypothesis} for a pattern $\patt$:
\begin{equation*}
\Pr_{(s,\ell) \sim \probdist} \pars{ \patt \in s \wedge \ell = 1 } = \Pr_{(s,\ell) \sim \probdist} \pars{ \patt \in s } \Pr_{(s,\ell) \sim \probdist} \pars{ \ell = 1 } .
\end{equation*}
Since we are interested in patterns with a significant association with the target variable being $1$, we are interested in finding patterns for which the following \emph{alternative hypothesis} holds:
\begin{equation*}
\Pr_{(s,\ell) \sim \probdist} \pars{ \patt \in s \wedge \ell = 1 } > \Pr_{(s,\ell) \sim \probdist} \pars{ \patt \in s } \Pr_{(s,\ell) \sim \probdist} \pars{ \ell = 1 }.
\end{equation*}

To this end, we define the \emph{quality}\footnote{The quality $\tqual$ is often called \emph{leverage} for general patterns~\cite{hamalainen2019tutorial}, but we use the term \emph{quality} given its relation to the $1$-quality commonly employed to find interesting subgroups.} of a pattern $\patt$ as 
\begin{equation*}
\tqual =  \Pr_{(s,\ell) \sim \probdist} \pars{ \patt \in s \wedge \ell = 1 } -  \Pr_{(s,\ell) \sim \probdist} \pars{ \patt \in s } \Pr_{(s,\ell) \sim \probdist} \pars{ \ell = 1 }.
\end{equation*}
Note that the alternative hypothesis is equivalent to 
$ \tqual > 0$, and the null hypothesis is equivalent to $\tqual=0$.  Therefore, finding significant patterns is equivalent to finding patterns with quality $ \tqual > 0$. 

For example, consider the study of the association between the characteristics of the users of an online social network and users' interests in a given topic. In this case, each user is a transaction, users' characteristics are the features, and being interested or not in the topic defines the target variable. Significant patterns in this example are associations between users' characteristics and users' interests that are significantly stronger than  expected under the null hypothesis of independence between characteristics and interests. For example, if the probability that a user has a given binary feature $f$ is $0.3$, and the probability that a user is interested in the topic is $0.5$, then under the null hypothesis of independence we have that the probability that a user has feature $f$ \emph{and} is interested in the topic is $0.15$. Therefore, the feature $f$ is \emph{significantly associated} with the topic of interest if the \emph{actual} probability that a user has feature $f$ \emph{and} is interested in the topic is $> 0.15$. 

\textit{Task Definition.} Given a dataset $\dataset$ and the corresponding pattern language $\lang$, our goal is to identify significant patterns, that is,  patterns $\patt \in \lang$ with quality $ \tqual > 0$. Since the distribution $ \probdist$ is unknown and we have access only to the dataset $\dataset$ comprising transactions sampled from $\probdist$, we cannot hope to discover \emph{all} significant patterns without errors. As a consequence, we must resort to approximations. 
In particular, define the subset $\lang^\star$ of the language $\lang$ of patterns for which the null hypothesis holds:
\begin{align*}
\lang^\star &= \brpars{ \patt \in \lang : \tqual = 0 } .
\end{align*}
Our task is then to produce a  subset $O \subseteq \lang$ of all patterns in the language with \emph{Family-Wise Error Rate} (FWER) below a user-defined threshold $\delta$, such that the probability that $O$ contains at least one element from $\lang^\star$ is at most $\delta$:
\begin{equation}
\Pr_{\dataset}\pars{O \cap \lang^\star \neq \emptyset} \le \delta. \label{eq:fwerdef}
\end{equation}
Note that \eqref{eq:fwerdef} implies that $O$ is \emph{false discovery free} approximation, which have previously been considered for other pattern mining tasks~\cite{riondato2020misosoup,santoro2020mining}.

Since $\tqual$ depends on the unknown distribution $\probdist$, we define a statistic $\qualp$ that corresponds to an estimate of $\tqual$ from data and that we will use in our algorithm \algname\ to identify significant patterns. 
For each pattern $\patt$, we define the functions $f_\patt$ and $g_\patt$ as follows: 
each $f_\patt : \X \rightarrow \{ 0 , 1 \}$ is defined as $f_\patt(s) = \ind{ \patt \in s }$, such that $f_\patt(s) = 1$ if  $\patt \in s$, and $f_\patt(s) =0$ otherwise; 
each $g_\patt : \X \times \{0,1\} \rightarrow [-\mu(\dataset) , 1 - \mu(\dataset) ]$ is defined as
$g_\patt(s,\ell) = f_\patt(s) ( \ell - \mu(\dataset) )$.  Then, the estimate $\qualp$ of $\tqual$ for pattern $\patt$ from the dataset $\dataset$ is
\begin{align*}
\qualp = \frac{1}{m} \sum_{i=1}^m g_\patt(s_i,\ell_i) .
\end{align*}
Interestingly, $\qualp$ corresponds to the $1$-quality commonly used in subgroup mining \cite{atzmueller2015subgroup} 
\ifextversion
(see \Cref{sec:qualitymeasures}), 
\else
(see Appendix A.1), 
\fi
even if the way it is used in our algorithm \algname\ (see Section~\ref{sec:algo}) is different from its usual application, due to our focus on statistically significant patterns.

Intuitively, we expect significant patterns to have a sufficiently high empirical quality $\qualp$ measured on the data $\dataset$, since $\qualp$ estimates the (true) quality $\tqual$ of $\patt$. 
To take into account the multiple hypothesis testing issue described above, we need to identify a \emph{threshold} $\varepsilon$ such that reporting in output all patterns with quality $\geq \varepsilon$ has bounded FWER. 
Note that $\varepsilon$ should be as small as possible in order to have high statistical power (i.e., to report the largest set of results with guarantees on false discoveries). 
A critical quantity we study to address this issue is the \emph{supremum deviation} of the empirical qualities of non-significant patterns, defined as
\begin{align}
\sup_{\patt \in \lang^\star} \brpars{ \qualp } . \label{eq:supdev}
\end{align}
We address this challenge with \algname: we derive novel analytical bounds on the concentration of the supremum deviation \eqref{eq:supdev}, and build on these tools an efficient few-shot resampling algorithm to sharply estimate it. \algname\ achieves high statistical power while scaling to large datasets and complex languages.
Our approach applies to both conditional and unconditional testing, both of great interest in data mining.  

\subsection{Conditional and Unconditional Testing}
\label{sec:cond_vs_uncond}

When assessing the statistical significance of a pattern $\patt$, one has to choose between \emph{conditional} and \emph{unconditional} tests. A conditional test
assumes that the data-generating process represented by the unknown distribution $\probdist$ only produces datasets with $m$ transactions in which both the frequency $\freqp$ of pattern $\patt$ and the fraction $\mu(\dataset)$ of transactions with target value $1$ are the same as in the observed dataset; that is, it \emph{conditions} on the observed variables of interest. In contrast, \emph{unconditional tests} assume that $\freqp$ and $\mu(\dataset)$ are the realization of corresponding random variables. Unconditional tests therefore assess the association between a pattern and class labels considering also scenarios (i.e., datasets) where all frequencies of the patterns and/or the average target value may differ from what is observed in the data.
Equivalently, conditional tests and unconditional tests are based on different assumptions regarding how data is generated and collected, namely, whether the variables of interest 
would be the same in a different repetition of the experiment (conditional tests) or not (unconditional tests).

Consider for example the scenario of online social networks described in \Cref{sec:sig_sub}, and assume for simplicity that we are interested in associations between the single features and the target. If the data is collected so that the total number of transactions for each value of the target is fixed and the fraction of transactions with given values of the features is fixed as well, then conditional testing is more appropriate. If instead one collects the data without constraints on the features/target values (e.g., simply collecting as many transactions as possible), then unconditional testing is more appropriate.

Conditional tests and unconditional tests are therefore both valid and of interest for data mining applications, and the choice between the two classes depends on the specific scenario, even if in practice conditional tests are usually preferred for computational reasons, since unconditional tests need to take into account more uncertainties in the observed quantities.
In what follows, we introduce a general algorithm to identify significant patterns for both conditional testing and unconditional testing. Our algorithm is extremely efficient in both cases due to the use of few-shot resampling.

\section{\algname\ Algorithm}
\label{sec:algo}

We now describe our algorithm~\algname\ (Algorithm~\ref{algo:main}) to find significant patterns for both conditional testing and unconditional testing. We first present the general approach, that is common to both testing scenarios, and then present the details for conditional testing in Section~\ref{sec:algo_conditional}, and the details for unconditional testing in Section~\ref{sec:algo_unconditional}.

{\sloppy
In a nutshell,  \algname\ identifies significant patterns from a dataset $\dataset = \brpars{ (s_1 , \ell_{1}) , \dots , (s_m , \ell_{m}) }$ using a few-shot resampling approach to compute rigorous probabilistic bounds to the deviation of the estimated qualities $\qualp$ of the patterns, under the null hypothesis of no association between the patterns and the target label. 
It then reports in output all patterns with estimated quality $\qualp$ above such deviation.
}

\algname\ considers a collection  $\resset = \{\dataset^\star_1 , \dots, \dataset^\star_c\}$  of $c \geq 1$ i.i.d. \emph{resampled datasets}, each obtained by \emph{resampling} the target labels of $\dataset$ while maintaining the same features of the transactions of $\dataset$. 
Each resampled dataset is $\dataset^\star_j = \brpars{ (s_1 , \xi_{1,j}) , \dots , (s_m , \xi_{m,j}) }$, where
$\xi_{i,j}$ are i.i.d. random variables with 
\begin{align*}
\xi_{i,j} \sim Bern(p) 
, \forall i \in [1,m] , \forall j \in [1,c] ,
\end{align*}
and $Bern(p)$ is the Bernoulli random variable with parameter $p$ (i.e., it is $1$ with probability $p$, and  $0$ otherwise).

We now describe the general approach followed by \algname\ (Algorithm~\ref{algo:main}). 
\algname\ starts by computing an upper bound $\varepsilon_T$ to the deviation between the average target value $\mu(\dataset)$  observed in $\dataset$ and its expected value $\mu = \E_{\dataset} \sqpars{ \mu(\dataset) }$ under the null hypothesis (line~\ref{line1}). 
Note that 
$\mu$ is the probability that a sample from the unknown distribution $\probdist$ has target label $\ell$ equal to $1$, that is $\mu = \Pr_{(s,\ell)} \pars{ \ell = 1 }$. 
The  computation of $\varepsilon_T$ is performed by the procedure \texttt{boundTarget} and it depends on whether one is interested in conditional testing or in unconditional testing.
The  details of \texttt{boundTarget} for the two settings are described in Section~\ref{sec:algo_conditional} and in Section~\ref{sec:algo_unconditional}, respectively.
Then \algname\ uses $\varepsilon_T$ to obtain the upper bound $\hat{\mu}$ (line~\ref{line2}) and lower bound $\check{\mu}$ (line~\ref{line3}) to $\mu$ 
(we trivially assume $0 \leq \check{\mu} \leq \hat{\mu} \leq 1$). 
It then uses the procedure \texttt{resampleTarget} to generate $c$ resampled datasets $\resset$ (line~\ref{line4}) assigning to each transaction in the datasets of $\resset$ the target label $1$ with probability $p = \hat{\mu}$  (i.e, the upper bound to $\mu$).  
Note that the \texttt{resampleTarget} procedure is the same for both conditional and unconditional testing.
The algorithm then computes, from each of the $c$ resampled datasets of $\resset$, an estimate of the maximum deviation of the empirical quality $\qualp$ for non-significant patterns. 
This is achieved by computing, for every resampled dataset $\dataset^\star_j$, the quantity 
$\sup_{\patt \in \lang} \hsqual{\patt}{\dataset^\star_j}{\check{\mu}}$, 
where 
$\hsqual{\patt}{\dataset^\star_j}{\check{\mu}} = \frac{1}{m} \sum_{i=1}^m   f_\patt(s_i) ( \xi_{i,j} - \check{\mu})$. 
The value $\sup_{\patt \in \lang} \hsqual{\patt}{\dataset^\star_j}{\check{\mu}}$ can be interpreted as an empirical estimate of the maximum quality of non-significant patterns, measured from datasets sampled from the null distribution. 
Note that we use $\hsqual{\patt}{\dataset^\star_j}{\check{\mu}}$ as $\check{\mu}$ is a lower bound to $\mu$; 
consequently, $\hsqual{\patt}{\dataset^\star_j}{\check{\mu}}$ provides a proper \emph{upper bound} to such maximum deviation. 
Then, note that it is necessary to consider the supremum over the language $\lang$, since the set of true null hypothesis $\lang^\star$ is unknown. 
We remark that the computation of $\sup_{\patt \in \lang} \hsqual{\patt}{\dataset^\star_j}{\check{\mu}}$ can be performed with fast pattern enumeration strategies, similar to the ones leveraged by previous methods for frequent and significant pattern mining; 
in fact, \algname\ can be combined with any efficient exploration procedure, such as the ones that explore the search space of the pattern language of interest using pruning bounds, for instance depth-first~\cite{minato2014fast,llinares2015fast,terada2015high} or best-first searches~\cite{pietracaprina2007efficient,pellegrina2020efficient,pellegrina2022mcrapper}.  
The algorithm then stores the empirical deviations $\sup_{\patt \in \lang} \hsqual{\patt}{\dataset^\star_j}{\check{\mu}}$ in the variables $d_j$. 
Then, the average maximum deviation $\tilde{d}(\resset , \check{\mu})$ over the $c$ resampled datasets is computed (line~\ref{line6}) as the mean of the values $\{ d_j , j \in [1,c] \}$, and it is then used to compute a rigorous upper bound $\varepsilon$ to the supremum deviation $\sup_{\patt \in \lang^\star} \brpars{ \qualp  }$ (Eq.~\eqref{eq:supdev}) under the null hypothesis (line~\ref{line7}) through the procedure 
\texttt{boundStatistic}. 
The implementation of this procedure, i.e., the returned value of $\varepsilon$, depends on whether conditional testing or unconditional testing is considered, and it is described in Section~\ref{sec:algo_conditional} and Section~\ref{sec:algo_unconditional}, respectively. 
In both cases, we use advanced concentration bounds~\cite{boucheron2013concentration,mcdiarmid1989method} that allow us to obtain small values of $\varepsilon$ with a small number $c$ of resampled datasets. 
Finally, the set of patterns with estimated quality $\qualp$ greater than $\varepsilon + \varepsilon_T \freqp$ is reported in output (lines~\ref{line8}-\ref{line9}).

\begin{algorithm}[htb]
\SetNoFillComment%
  \KwIn{Pattern language $\lang$; dataset $\dataset$ of $m$ transactions; $c \geq1$; $\delta \in (0,1)$.}
  \KwOut{\mbox{Set $O \subseteq \lang$ of significant patterns with FWER $\le \delta$.}} 
  $\varepsilon_T \gets$ \texttt{boundTarget}($\mu(\dataset)$, $m$, $\delta$)\label{line1}\;
  $\hat{\mu} \gets  \mu(\dataset) + \varepsilon_T$\label{line2}\;
  $\check{\mu} \gets  \mu(\dataset) - \varepsilon_T$\label{line3}\;
  $\resset \gets $ \texttt{resampleTarget($\dataset , c , \hat{\mu}$)}\label{line4}\;
  \lForAll{$j \in [1 , c]$}{$d_j \gets \sup_{\patt \in \lang} \bigl\{ \hsqual{\patt}{\dataset^\star_j}{\check{\mu}} \bigr\}$ \label{line5}} 
  $\tilde{d}(\resset , \check{\mu}) \gets \frac{1}{c} \sum_{j=1}^{c} d_j $\label{line6}\;
 $\varepsilon \gets$ \texttt{boundStatistic}($\dataset$, $\lang$, $\tilde{d}(\resset, \check{\mu})$, $\delta$)\label{line7}\;
  $O \gets \brpars{ \patt \in \lang : \qualp \geq \varepsilon + \varepsilon_T \freqp }$\label{line8}\;
  \textbf{return} $O$\label{line9}\;
  \caption{\algname}\label{algo:main}
\end{algorithm}

Note that while the computation of $\varepsilon_T$ and $\varepsilon$ depends on whether conditional testing or unconditional testing is considered, the overall approach followed by \algname\ is the same in both cases. In particular, for both cases \algname\ relies on the resampled datasets $\resset$ to estimate the maximum deviation, over all patterns $\patt \in \lang$, of the estimate $\qualp$ from the (unknown) value $\tqual$ under the null hypothesis. Moreover, the exact same procedure is used to generate the resampled datasets. 
Note that our approach is similar to permutation approaches for significant pattern mining, which use permuted datasets to estimate the significance of patterns, but with two crucial differences. 
First, our datasets are obtained by resampling the target values, and not by permuting them, which allows us to obtain rigorous bounds for both conditional and unconditional testing, while permutation approaches can be used for conditional testing only. 
Second, since our analysis depends on the \emph{expectation} of the maximum deviation, we can employ advanced bounds on the concentration of the expected value of functions of independent random variables; this allows us to use a small number $c$ of resampled datasets, as shown by our analysis and experimental evaluation. This is in contrast with permutation approaches (e.g., the ones based on WY permutation testing~\cite{llinares2015fast,pellegrina2020efficient,terada2015high}) that instead estimate the quantiles of the distribution of the maximum deviation using a large number of permutations 
\ifextversion
(see also Section~\ref{sec:appendix_comp_perm} in Appendix for a more detailed comparison).
\else
(see also Section~A.2 in Appendix for a more detailed comparison).
\fi

\subsection{\algname\ for Conditional Testing}
\label{sec:algo_conditional}

\sloppy{
We now describe the details of procedures 
\texttt{boundTarget} 
and 
\texttt{boundStatistic} 
for the version of \algname\ that uses conditional testing, which we refer to as \algnamec. 
As a reminder, a conditional test for our problem assumes that the average target value $\mu$ and patterns frequencies $\freqp$ are fixed, for all $\patt \in \lang$, to the values observed in the dataset $\dataset$.
}

For 
\texttt{boundTarget}, 
since the average target value $\mu$ is fixed to the value $\mu(\dataset)$ observed in the dataset $\dataset$, the bound $\varepsilon_T$ on its deviation from the expectation is $0$, that is, 
\texttt{boundTarget} 
simply returns $0$. Note that this implies that the output of \algnamec\ consists of all patterns in $\lang$ with $ \qualp \geq \varepsilon$ (line~\ref{line8}).

For 
\texttt{boundStatistic}, 
we now show how to compute a rigorous probabilistic bound $\varepsilon$ to the supremum deviation $\sup_{\patt \in \lang^\star} \brpars{ \qualp  }$ (Eq.~\eqref{eq:supdev}) under the null hypothesis. 
The bound is obtained by computing the average maximum deviation $\tilde{d}(\resset, \check{\mu})$ over $c$ resampled datasets, and then applying advanced concentration results~\cite{boucheron2013concentration}. Note that since $\varepsilon_T = 0$,  in this case $\tilde{d}(\resset, \check{\mu})= \tilde{d}(\resset, \mu(\dataset))$.

Since \algnamec\ is based on resampled datasets, where each target label is sampled independently, our estimate  $\tilde{d}(\resset, \mu(\dataset))$  of the expected maximum deviation is not based on the conditional distribution assumed by conditional tests, where the fraction of target labels equal to $1$ is exactly $\mu(\dataset) $ in \emph{every} dataset (while in our resampled datasets such fraction may vary, and it is equal to $\mu(\dataset)$ only in expectation). 
We therefore need to relate the supremum deviation of the observed quality of patterns on resampled datasets with the one observed in datasets sampled from the conditional distribution. 
Interestingly, we show that the resampling and conditional distributions are \emph{closely related}, in the sense that high probability bounds for the former also apply to the latter. 

We first prove a general result (\Cref{thm:expectupperboundresamples}) relating the expectation of monotone functions of permutations of binary vectors, corresponding to the conditional distribution, with the expectation taken w.r.t. to independent resamples, corresponding to resampled datasets. For an integer $k$ with $0 \leq k \leq m$, define the set $B(k)$ of binary vectors with $k$ entries equal to one as
\begin{align*}
B(k) = \Bigl\{ \mathbf{v} \in \{0 , 1\}^m , \sum_{i=1}^m \mathbf{v}_i = k \Bigr\} ,
\end{align*}
and let $U(B(k))$ be the uniform distribution over the set $B(k)$.
Equivalently, $U(B(k))$ corresponds to the set of uniform permutations of a binary vector with $k$ ones.
Then, define $I(p)$ as a probability distribution over $\{ 0,1 \}^m$, such that each entry $\mathbf{v}_i$ of a random vector $\mathbf{v}$ taken from $I(p)$ is an i.i.d. Bernoulli r.v. with $\Pr(\mathbf{v}_i = 1) = p = 1- \Pr(\mathbf{v}_i = 0)$, for some $p \in [0,1]$.
\begin{lemma}
\label{thm:expectupperboundresamples}
Let $f : \{0 , 1\}^m \rightarrow \R$ be a nonnegative function such that $\E_{\mathbf{v} \sim U(B(k))} \sqpars{ f(\mathbf{v}) }$ is either monotonically increasing or monotonically decreasing in $k$. 
It holds
\begin{align*}
\E_{\mathbf{v} \sim U(B(k))} \sqpars{ f(\mathbf{v}) } \leq 2 \E_{\mathbf{v} \sim I(k/m)} \sqpars{ f(\mathbf{v}) } .
\end{align*}
\end{lemma}
To prove \Cref{thm:expectupperboundresamples}, we need the following technical result,
providing bounds to the probability that a Binomial random variable exceeds its expectation \cite{jogdeo1968monotone}.
\begin{lemma}
\label{thm:binomiallowerbound}
Let $\mu m$ be an integer. Then it holds
\begin{align*}
\Pr_{\mathbf{v} \sim I(\mu)} \pars{ \sum_{i=1}^m \mathbf{v}_i > \mu m } < \frac{1}{2} < \Pr_{\mathbf{v} \sim I(\mu)} \pars{ \sum_{i=1}^m \mathbf{v}_i \geq \mu m } .
\end{align*}
\end{lemma}
We now prove \Cref{thm:expectupperboundresamples}. 
\begin{proof}[Proof of \Cref{thm:expectupperboundresamples}]
We prove the result assuming that $\E_{\mathbf{v} \sim U(B(k))} \sqpars{ f(\mathbf{v}) }$ is monotonically increasing in $k$, as the other case is analogous. 
First, we note that 
\begin{align*}
\E_{\mathbf{v} \sim U(B(k))} \sqpars{ f(\mathbf{v}) } = \E_{\mathbf{v} \sim I(k/m)} \sqpars{ f(\mathbf{v}) \mid \sum_{i=1}^m \mathbf{v}_i = k } .
\end{align*}
Therefore, we have
\begin{align*}
& \E_{\mathbf{v} \sim I(k/m)} \sqpars{ f(\mathbf{v}) } \\
&= \sum_{j=0}^m \E_{\mathbf{v} \sim I(k/m)} \sqpars{ f(\mathbf{v}) \mid \sum_{i=1}^m \mathbf{v}_i = j } \Pr_{\mathbf{v} \sim I(k/m)} \pars{ \sum_{i=1}^m \mathbf{v}_i = j } \\
& \geq  \sum_{j=k}^m \E_{\mathbf{v} \sim I(k/m)} \sqpars{ f(\mathbf{v}) \mid \sum_{i=1}^m \mathbf{v}_i = j } \Pr_{\mathbf{v} \sim I(k/m)} \pars{ \sum_{i=1}^m \mathbf{v}_i = j }\\
& \geq \E_{\mathbf{v} \sim I(k/m)} \sqpars{ f(\mathbf{v}) \mid \sum_{i=1}^m \mathbf{v}_i = k } \sum_{j=k}^m  \Pr_{\mathbf{v} \sim I(k/m)} \pars{ \sum_{i=1}^m \mathbf{v}_i = j } \\
& = \E_{\mathbf{v} \sim U(B(k))} \sqpars{ f(\mathbf{v}) }  \Pr_{\mathbf{v} \sim I(k/m)} \pars{ \sum_{i=1}^m \mathbf{v}_i \geq k } \\
& \geq \E_{\mathbf{v} \sim U(B(k))} \sqpars{ f(\mathbf{v}) } \frac{1}{2} ,
\end{align*}
where the last inequality follows from \Cref{thm:binomiallowerbound}.
\end{proof}
We make use of \Cref{thm:expectupperboundresamples} to prove the following result.
\begin{theorem}
\label{thm:probupperboundcond}
Define the constant $\bar{\mu} = \mu(\dataset)$ and let $k = \bar{\mu} m$. 
For any $z \geq 0$, it holds
\begin{align*}
&\Pr_{\mathbf{v} \sim U(B(k))} \pars{ \sup_{\patt \in \lang^\star} \frac{1}{m} \sum_{i=1}^m f_{\patt}(s_i)( \mathbf{v}_i - \bar{\mu} ) \geq z } \\ 
&\leq 2 \Pr_{\mathbf{v} \sim I(\bar{\mu})} \pars{ \sup_{\patt \in \lang^\star} \frac{1}{m} \sum_{i=1}^m f_{\patt}(s_i)( \mathbf{v}_i - \bar{\mu} ) \geq z }. 
\end{align*}
\end{theorem}

\begin{proof}[Proof of \Cref{thm:probupperboundcond}]
Define the function $g : \{0 , 1\}^m \rightarrow \{0,1\}$ 
\begin{align*}
g(\mathbf{v}) = \ind{ \sup_{\patt \in \lang^\star} \frac{1}{m} \sum_{i=1}^m f_{\patt}(s_i)( \mathbf{v}_i - \bar{\mu} ) \geq z } .
\end{align*}
In order to apply \Cref{thm:expectupperboundresamples}, 
we prove that 
$\E_{\mathbf{v} \sim U(B(k))} \sqpars{ g(\mathbf{v}) }$ is nondecreasing with $k$. 
This is equivalent to show that 
\begin{align*}
\E_{\mathbf{v} \sim U(B(k))} \sqpars{ g(\mathbf{v}) } \leq \E_{\mathbf{v} \sim U(B(k^\prime))} \sqpars{ g(\mathbf{v}) },
\end{align*}
where $k^\prime = k+j$, 
for any pair of integers $k \in [0,m]$ and $j \in [0,m-k]$. 
To do so, we build a coupling $\pi(k,k^\prime)$ between the two distributions $U(B(k))$ and $U(B(k^\prime))$ as follows.
For any $\mathbf{v}$ taken from $U(B(k))$, define $\mathbf{v}^\prime$ as a copy of $\mathbf{v}$ (i.e., such that $\mathbf{v}_i = \mathbf{v}^\prime_i, \forall i \in [1,m]$), that is modified according to the following procedure: for $j$ times, select uniformly at random an index $i$ such that $\mathbf{v}^\prime_i = 0$, and set $\mathbf{v}^\prime_i$ to $1$. 
Denote the pair $\mathbf{v},\mathbf{v}^\prime$ sampled according to $\pi(k,k^\prime)$ as the output of this procedure. 
It is immediate to observe that $\mathbf{v}^\prime \sim U(B(k^\prime))$ (i.e., that the marginal distribution of $\mathbf{v}^\prime$ is $U(B(k^\prime)$), and that $\mathbf{v}_i \leq \mathbf{v}^\prime_i, \forall i \in [1,m]$. 
This implies that, for any pair of vectors $\mathbf{v},\mathbf{v}^\prime \sim \pi(k,k^\prime)$, it holds 
\begin{align*}
\sup_{\patt \in \lang^\star} \frac{1}{m} \sum_{i=1}^m f_{\patt}(s_i)( \mathbf{v}_i - \bar{\mu} ) 
\leq 
\sup_{\patt \in \lang^\star} \frac{1}{m} \sum_{i=1}^m f_{\patt}(s_i)( \mathbf{v}^\prime_i - \bar{\mu} ). 
\end{align*}
A consequence of this fact is 
\begin{align*}
&\E_{\mathbf{v}^\prime \sim U(B(k+j))} \sqpars{ g(\mathbf{v}^\prime) } \\
&= \Pr_{\mathbf{v}^\prime \sim U(B(k+j))} \pars{ \sup_{\patt \in \lang^\star} \frac{1}{m} \sum_{i=1}^m f_{\patt}(s_i)( \mathbf{v}^\prime_i - \bar{\mu} ) \geq z } \\
&= \Pr_{\mathbf{v} , \mathbf{v}^\prime \sim \pi(k,k^\prime)} \pars{ \sup_{\patt \in \lang^\star} \frac{1}{m} \sum_{i=1}^m f_{\patt}(s_i)( \mathbf{v}^\prime_i - \bar{\mu} ) \geq z } \\
& \geq \Pr_{\mathbf{v} , \mathbf{v}^\prime \sim \pi(k,k^\prime)} \pars{ \sup_{\patt \in \lang^\star} \frac{1}{m} \sum_{i=1}^m f_{\patt}(s_i)( \mathbf{v}_i - \bar{\mu} ) \geq z } \\
& = \Pr_{\mathbf{v} \sim U(B(k))} \pars{ \sup_{\patt \in \lang^\star} \frac{1}{m} \sum_{i=1}^m f_{\patt}(s_i)( \mathbf{v}_i - \bar{\mu} ) \geq z } \\
& = \E_{\mathbf{v} \sim U(B(k))} \sqpars{ g(\mathbf{v}) } .
\end{align*}
Therefore, $g$ is nonnegative and $\E_{\mathbf{v} \sim U(B(k))} \sqpars{ g(\mathbf{v}) }$ is nondecreasing in $k$. 
We apply \Cref{thm:expectupperboundresamples} to the function $g$, obtaining the statement.
\end{proof}

We note that Theorem~\ref{thm:probupperboundcond} is precisely what we seek: it implies that 
large supremum deviations that are unlikely in the independent resamples distribution $\mathbf{v} \sim I(\bar{\mu})$ are also unlikely in the conditional distribution $\mathbf{v} \sim U(B(k))$. 
By using Theorem~\ref{thm:probupperboundcond}, we prove the following result, that implies strong concentration of the supremum deviation of pattern qualities w.r.t. their expectations, taken w.r.t. independent resamples of the target labels rather than permutations.

\begin{theorem}
\label{thm:tailboundconditional}
Define $\bar{\mu} = \mu(\dataset)$, and 
\begin{align*}
\omega = (1-\bar{\mu}) \min \Bigl\{ \bar{\mu} , \sup_{\patt \in \lang} \frac{1}{m} \sum_{i=1}^m f_{\patt}(s_i) \Bigr\}. 
\end{align*}
With probability at least $1 - \delta$ over $\mathbf{v} \sim U(B(k))$, it holds
\begin{align*}
&\sup_{\patt \in \lang^\star} \frac{1}{m} \sum_{i=1}^m f_{\patt}(s_i)( \mathbf{v}_i - \bar{\mu} )  \\
&\leq \E_{\mathbf{v} \sim I(\bar{\mu})} \sqpars{ \sup_{\patt \in \lang^\star} \frac{1}{m} \sum_{i=1}^m f_{\patt}(s_i)( \mathbf{v}_i - \bar{\mu} ) } + \sqrt{ \frac{ 2 \omega \log\left(\frac{2}{\delta}\right)}{m} } .
\end{align*}
\end{theorem}
To prove \Cref{thm:tailboundconditional}, we need the following technical result regarding the concentration of functions of independent random variables.

\begin{theorem}[Theorem 6.7 of \cite{boucheron2013concentration}]
\label{thm:tailboundfindepvars}
Let $g : \mathcal{Y}^m \rightarrow \R$ be a function, and let $X = ( X_1 , \dots , X_m ) \in \mathcal{Y}^m$ a collection of $m$ independent random variables.
Define $\bar{X}^j = ( X_1 , \dots , X_{j-1} , X^\prime_j , X_{j+1} , \dots , X_m ) \in \mathcal{Y}^m$ as a copy of $X$, where its $j$-th element $X_j$ is replaced by an independent copy $X^\prime_j$.
Assume that, for some constant $q \geq 0$,  
\begin{align*}
 \sum_{j=1}^m \E\sqpars{ \pars{ g(X) - g(\bar{X}^j) }_+^2 \given X } \leq q 
\end{align*}
holds almost surely. 
Then, it holds, for all $t\geq0$, 
\begin{align*}
\Pr 	\Bigl( g(X) \geq \E_X\sqpars{ g(X) } + t \Bigr) \leq \exp ( -t^2/(2q) ) .
\end{align*}
\end{theorem}

We use these results to prove \Cref{thm:tailboundconditional}. 

\begin{proof}[Proof of \Cref{thm:tailboundconditional}]
Define the function $g : \{ 0 , 1\}^m \rightarrow \R$ as
\begin{align*}
g(\mathbf{v}) = \sup_{\patt \in \lang^\star} \frac{1}{m} \sum_{i=1}^m f_{\patt}(s_i)( \mathbf{v}_i - \bar{\mu} ) .
\end{align*}
We first note that, from the result of \Cref{thm:probupperboundcond}, it is sufficient to show that, choosing the constants $t$ and $z$ as
\begin{align}
t =  \sqrt{ \frac{ 2 \omega \log(\frac{2}{\delta})}{m} } , \;\;\; z = \E_{\mathbf{v} \sim I(\bar{\mu})} \sqpars{ g(\mathbf{v}) } + t ,
\label{eq:valueoft}
\end{align}
it holds
\begin{align*}
\Pr_{\mathbf{v} \sim I(\bar{\mu})} \pars{ g(\mathbf{v}) \geq z } \leq \delta/2.
\end{align*}
To do so, we make use of \Cref{thm:tailboundfindepvars}. 
Define $\bar{\mathbf{v}}^j$ a copy of $\mathbf{v}$, where its $j$-th element $\mathbf{v}_j$ is replaced by an independent copy $\mathbf{v}^\prime_j$. 
We want to upper bound 
\begin{align*}
\sum_{j=1}^m \E\sqpars{ \pars{ g(\mathbf{v}) - g(\bar{\mathbf{v}}^j) }_+^2 \given \mathbf{v} }
\end{align*}
below some constant $q \geq 0$. 
Define $\patt^\star$ as one of the elements of $\lang^\star$ that achieve the supremum in $g(\mathbf{v})$. 
We observe that 
\begin{align*}
&\pars{ g(\mathbf{v}) - g(\bar{\mathbf{v}}^j) }_+ \\
&\leq 
\biggl( \frac{1}{m} \sum_{i=1}^m f_{\patt^\star} (s_i) (\mathbf{v}_i - \bar{\mu}) - \frac{1}{m} \sum_{i=1}^m f_{\patt^\star} (s_i) (\bar{\mathbf{v}}^j_i - \bar{\mu}) \biggr)_+ \\
&= 
\biggl( \frac{1}{m} f_{\patt^\star} (s_j) (\mathbf{v}_j - \bar{\mu}) - \frac{1}{m} f_{\patt^\star} (s_j) (\mathbf{v}^\prime_j - \bar{\mu}) \biggr)_+ \\
&= 
\biggl( \frac{1}{m} f_{\patt^\star} (s_j) (\mathbf{v}_j - \mathbf{v}^\prime_j) \biggr)_+ . \numberthis \label{eq:proofzerocond}
\end{align*}
We now observe that $\mathbf{v}^\prime_j = 0$ is the only possible value of $\mathbf{v}^\prime_j$ that makes \eqref{eq:proofzerocond} be $> 0$. Therefore, we have
\begin{align*}
&\sum_{j=1}^m \E\sqpars{ \pars{ g(\mathbf{v}) - g(\bar{\mathbf{v}}^j) }_+^2 \given \mathbf{v} } \\
&\leq \sum_{j=1}^m \E\sqpars{ \pars{ \frac{1}{m} f_{\patt^\star} (s_j) (\mathbf{v}_j - \mathbf{v}^\prime_j) }_+^2 \given \mathbf{v} } \\
&= (1-\bar{\mu}) \sum_{j=1}^m \pars{ \frac{1}{m} f_{\patt^\star} (s_j) \mathbf{v}_j }^2  \\
&= \frac{(1-\bar{\mu})}{m} \sum_{j=1}^m \pars{ \frac{1}{m} f_{\patt^\star} (s_j) \mathbf{v}_j }  \\
&\leq \frac{(1-\bar{\mu})}{m} \sup_{\patt \in \lang^\star} \sum_{j=1}^m \pars{ \frac{1}{m} f_{\patt} (s_j) \mathbf{v}_j }  
\leq \frac{\omega}{m}  .
\end{align*}
We apply \Cref{thm:tailboundfindepvars} to the function $g$ with $q = \omega/m$, obtaining that
\begin{align*}
\Pr_{\mathbf{v} \sim I(\bar{\mu})} \Bigl( g(\mathbf{v}) \geq \E_{\mathbf{v} \sim I(\bar{\mu})}[g(\mathbf{v})] + t \Bigr) \leq \exp ( - m t^2 / (2 \omega ) ) .
\end{align*}
Setting $t$ as in \eqref{eq:valueoft}, it is immediate to observe that the probability above is $\leq \delta/2$, obtaining the statement.  
\end{proof}

Note that Theorem~\ref{thm:tailboundconditional} provides a probabilistic upper bound, holding with probability at least $1-\delta$, to the maximum observed value of the pattern quality $\sup_{\patt \in \lang^\star} \frac{1}{m} \sum_{i=1}^m f_{\patt}(s_i)( \mathbf{v}_i - \bar{\mu} )$ when the conditional distribution $\mathbf{v} \sim U(B(k))$ is considered, in terms of the expectation of the maximum observed value of the pattern quality according to the resampled distribution $\mathbf{v} \sim I(\bar{\mu})$.
Then, the following result proves that the estimation $\tilde{d}(\resset, \check{\mu})$ of the expected deviation in the upper bound above is very accurate.
\begin{theorem}
\label{thm:upperboundestcond}
With probability at least $1-\delta/4$ over $\resset$, it holds
\begin{align*}
\E_{\mathbf{v} \sim I(\bar{\mu})} \sqpars{ \sup_{\patt \in \lang^\star} \frac{1}{m} \sum_{i=1}^m f_{\patt}(s_i)( \mathbf{v}_i - \bar{\mu} ) } \leq \tilde{d}(\resset, \check{\mu}) + \sqrt{ \frac{ \log \bigl( \frac{4}{\delta} \bigr) }{2cm} } .
\end{align*}
\end{theorem} 

By combining the results above, we prove the guarantees provided by \algname\ for the task of finding significant patterns when conditional testing is used. In particular, the following Corollary proves that,
when \texttt{boundStatistic} returns the value $\varepsilon$ defined below, then  the output set $O$ of significant patterns returned by \algname\ has FWER bounded by the user-defined parameter $\delta$.
\begin{corollary}
\label{thm:fwerguaranteesfsrc}
Fix $\delta \in (0,1)$ and $c \geq 1$. 
Let $O$ be the output of \algname\ with input parameters $\lang$, $\dataset$, $c$, $\delta$, and let 
\begin{align*}
\varepsilon = \tilde{d}(\resset, \check{\mu}) + \sqrt{ \frac{ 2 \omega \log \bigl( \frac{4}{\delta} \bigr) }{m} } + \sqrt{ \frac{ \log \bigl( \frac{4}{\delta} \bigr) }{2cm} }
\end{align*}
be the value returned by \texttt{boundStatistic} in line~\ref{line7}. 
Then, the set $O$ has FWER $\leq \delta$ under the conditional null distribution. 
\end{corollary}

Note that the value of $\varepsilon$ returned by 
\texttt{boundStatistic} 
requires to compute $\sup_{\patt \in \lang } \bigl\{ \hsqual{\patt}{\dataset^\star_j}{\check{\mu}} \bigr\}$, which is costly, but only needs to be performed on $c$ resampled datasets. As we will show in our experimental evaluation (see Section~\ref{sec:experiments}), small values of $c$ suffice. Moreover, the maximum frequency of a pattern in the language $\lang$ is required in order to compute $\omega$ (see Theorem~\ref{thm:tailboundconditional}). Such maximum frequency can be computed very efficiently in most data mining tasks. For example: in itemset mining it corresponds to the frequency of the most frequent item; in subgroup mining, it is equal to $1$ whenever a continuous feature is present and the conditions defining the pattern language $\lang$ include inequalities.

\textit{Power analysis.} 
The results above show that \algname\ rigorously controls the probability of \emph{false positives} (i.e., patterns for which the null hypothesis hold but are wrongly reported in output as significant). 
However, they do not provide guarantees on the \emph{power} of \algname, that is, its ability to report patterns $\patt$ with sufficiently high quality $\tqual$. 
The following result provides guarantees on the power of \algnamec, the version of \algname\ that uses conditional testing, for the pattern language of subgroups. 
Our analysis is based on a probabilistic upper bound to  $\tilde{d}(\resset , \check{\mu})$, obtained from bounds to the \emph{pseudodimension}~\cite{li2001improved,pollard2012convergence,ShalevSBD14} of subgroups,  
and an advanced concentration bound for sums of \emph{dependent random variables} \cite{dubhashi2009concentration}, that hold under mild (but necessary) assumptions on the distribution of alternative hypotheses (the set of patterns with $\tqual > 0$). 
More precisely, we assume that the target labels of the transactions that support 
a pattern $\patt$ with $\tqual > 0$ are distributed according to a \emph{noncentral hypergeometric distribution}~\cite{wallenius1963biased}, i.e., a biased version of the standard hypergeometric distribution. 
\ifextversion
We provide the proofs and additional details in \Cref{sec:app_power}. 
\else
We provide the proofs and additional details in Appendix~A.5. 
\fi

\begin{theorem}
\label{thm:power_conditional}
Fix $\delta \in (0,1)$ and $c , z \geq 1$. 
Let $O$ be the output of \algname\ with input parameters $\lang$, $\dataset$, $c$, and $\delta$,
where $\lang$ is the language of subgroups composed by conjunctions with at most $z$ conditions over $d$ continuous features. 
Then, with probability at least $1 - \delta$, $O$ contains \emph{all} patterns with quality $\tqual$ satisfying  
\begin{align*}
 \tqual \geq  & \; \sqrt{ \frac{ 2 \hat{\omega} z \ln(\frac{e^3 d m^2 }{4z^3}) }{ m } }  
 + \sqrt{ \frac{ \ln \bigl( \frac{2}{\delta} \bigr) }{ 2cm } } 
 + \frac{ z \ln \bigl( \frac{e^3 d m^2 }{4z^3} \bigr) }{ 3m } \\
 & + \sqrt{ \frac{ 2 \omega \log \bigl( \frac{4}{\delta} \bigr) }{m} } 
 +  \sqrt{ \frac{ \log \bigl( \frac{4}{\delta} \bigr) }{2cm} } 
 + \sqrt{ \frac{ 2 \hat{\mathsf{f}}(\dataset) z \ln(\frac{e^3 d m^2 }{2z^3\delta}) }{m} } , 
\end{align*}
where  
$\hat{\mathsf{f}}(\dataset) = \sup_{\patt \in \lang} \freqp$ and 
$\hat{\omega} = \hat{\mu} (1 - \check{\mu}) \hat{\mathsf{f}}(\dataset) $.
\end{theorem}

\subsection{\algname\ for Unconditional Testing}
\label{sec:algo_unconditional}

We now describe the details of procedures 
\texttt{boundTarget} 
and 
\texttt{boundStatistic} 
for the version of \algname\ that uses unconditional testing, which we refer to as \algnameu. As a reminder, in our scenario an unconditional test assumes that the average target value  and patterns frequencies $\freqp$ are observations of random variables, whose expected values are unknown. In particular, here we consider the transactions $\brpars{ (s_1 , \ell_1) , \dots , (s_m , \ell_m) }$ constituting the dataset $\dataset$ as \emph{i.i.d.} samples from an unknown distribution $\probdist$. 
This scenario is more complex than the conditional testing one, since we need to account for the unknown deviation of all observed values when computing a bound to the maximum pattern quality under the null hypothesis. However, we show that the use of resampled datasets allow us to efficiently take into account such deviations.

For 
\texttt{boundTarget}, 
recall that $\mu$ is the the probability that a sample from the unknown distribution $\probdist$ has target label $\ell$ equal to $1$, such that $\mu = \E_{\dataset} \sqpars{ \mu(\dataset) }  = \Pr_{(s,\ell) \sim \probdist} \pars{ \ell = 1 }$. 
More importantly, in this setting $\mu$ is unknown (as $\probdist$ is). 
However, given that the samples in $\dataset$ are i.i.d., the following result provides a probabilistic bound $\varepsilon_T$  to the deviation of the observed value $\mu(\dataset)$ from its expectation $\mu$.

\begin{lemma}
\label{thm:qualestimatemu}
Let $\dataset$ be a collection of $m$  samples taken i.i.d. from $\probdist$. 
For $\delta \in (0,1)$, it holds with probability $\geq 1 - \delta/4$
\begin{align*}
\abs{ \mu(\dataset) - \mu } \leq \varepsilon_T \doteq \sqrt{\frac{2 \min \bigl\{ \mu(\dataset) , \frac{1}{4} \bigr\} \ln \pars{ \frac{8}{\delta} } }{m}} + \frac{2 \ln \pars{ \frac{8}{\delta} } }{m}  . 
\end{align*}
\end{lemma}

\texttt{boundTarget}  
returns the value $\varepsilon_T$ as defined in Lemma~\ref{thm:qualestimatemu}, which requires the values $\mu(\dataset)$, $\delta$, and $m$.

For 
\texttt{boundStatistic}, 
the computation of the upper bound $\varepsilon$ to the maximum deviation between the observed pattern qualities and their expectations is more involved than in the conditional case. We show that the quality $\tqual$ of a pattern $\patt$, which depends on several unknown quantities (see Section~\ref{sec:sig_sub}), can be sharply estimated from a dataset $\dataset$ \emph{provided} that $\mu$ is known, which is \emph{not} the case for the unconditional distribution. We then show that using $\mu(\dataset)$ in place of $\mu$ provides a good estimate of $\tqual$, and prove a probabilistic upper bound on the difference between the two estimates. 

First, we introduce the function family that we use to analyze the supremum deviation of the empirical quality of patterns from the dataset $\dataset$.  
We define the family of functions $g^*_\patt : \X \times \{0,1\} \rightarrow [-\mu , 1 - \mu ]$, where $g^*_\patt$ is defined as
$g^*_\patt(s,\ell) = f_\patt(s) ( \ell - \mu )$.  Note that $g^*_\patt(s,\ell) = f_\patt(s) ( \ell - \mu )$ corresponds to the function $g_\patt(s,\ell)$ used by the \algname\ statistic  where the unknown value $\mu$ is used in place of its estimate $\mu(\dataset)$ obtained from dataset $\dataset$.
Define the estimator $\squalp$ of the quality $\tqual$ of $\patt$ from $\dataset$ as the average of $g^*_\patt$ over $\dataset$, that is  
\begin{align*}
\squalp = \frac{1}{m} \sum_{i=1}^m g^*_\patt(s_i,\ell_i) .
\end{align*}
Note that $\E_{\dataset}\sqpars{\squalp} = \tqual$, as $\squalp$ is an unbiased estimator  of $\tqual$. However, $\squalp$ depends on the unknown quantity $\mu$.
Even if $\squalp \neq \qualp$ (since $\mu(\dataset)$ may be $\neq \mu$), \algnameu\ exploits the fact that $\mu(\dataset)$ is sharply concentrated around $\mu$, as proved in Lemma~\ref{thm:qualestimatemu}. This implies that the maximum deviation $\sup_{\patt \in \lang}|\squalp - \qualp|$ for all patterns $\patt \in \lang$ can be sharply estimated from $\dataset$.
To this aim, we prove the following result.
\begin{theorem}
\label{thm:qualestimate}
Let $\dataset$ be a collection of $m$  samples taken i.i.d. from $\probdist$. 
For $\delta \in (0,1)$, with probability $\geq 1 - \delta/4$ it holds
\begin{align*}
\abs{ \squalp - \qualp } \leq \varepsilon_T \freqp , \forall \patt \in \lang .
\end{align*}
\end{theorem}
\begin{proof}
Using Lemma \ref{thm:qualestimatemu}, $\forall \patt \in \lang$ it holds with prob. $\geq 1 - \delta/4$ 
\begin{align*}
\abs{ & \squalp - \qualp } 
= \left\lvert \frac{1}{m} \sum_{i=1}^m g^*_\patt(s_i,\ell_i) - \frac{1}{m} \sum_{i=1}^m g_\patt(s_i,\ell_i)  \right\rvert \\
&= \left\lvert \frac{1}{m} \sum_{i=1}^m f_\patt(s_i) ( \mu(\dataset) - \mu ) \right\rvert 
\leq \varepsilon_T \freqp . \qedhere
\end{align*}
\end{proof}

We remark that the definition of $\squalp$ is crucial to the analysis of our algorithm \algnameu, as $\squalp$ is an average of $m$ independent random variables, while $\qualp$ is not (since it depends on $\mu(\dataset)$, that is estimated from the observations in the whole dataset). 

Recall the definition of the set $\lang^\star$ of non-significant patterns:
$\lang^\star = \brpars{\patt \in \lang : \tqual = 0 }$.
To output significant patterns while controlling the FWER below $\delta$, our goal is to bound the \emph{supremum deviation} $\sup_{\patt \in \lang^\star} \brpars{ \qualp  }$ (Eq.~\eqref{eq:supdev})   
below some value $\eta$ 
with probability at least $1-\delta$, for some $\delta \in (0,1)$, and provide in output all patterns with $\qualp \geq \eta$. 
To bound \eqref{eq:supdev}, we study the surrogate quantity
$\sup_{\patt \in \lang^\star} \brpars{ \squalp  } $,  
as \Cref{thm:qualestimate} guarantees a small bound on $\abs{ \squalp - \qualp }$.
Consider the collection 
$\resset = \brpars{\dataset^\star_1 , \dots, \dataset^\star_c}$ of $c \geq 1$ i.i.d. resampled datasets computed by \algname\ (line~\ref{line4}). 
We prove the following result. 
\begin{restatable}{theorem}{thmboundsupdev}
\label{thm:bound_dev}
Let $\dataset$ be a dataset of $m$ samples taken i.i.d. from a distribution $\probdist$,
and $\resset$ a collection of $c \geq 1$ i.i.d. resamples of the target labels of $\dataset$. 
For any $\delta \in (0,1)$, define $\nu_T \geq \mu (1 - \mu)$, $\nu \geq \nu_T \sup_{\patt \in \lang^\star} \brpars{ \E_\dataset\sqpars{ \freqp } }$, and  $\varepsilon$ as
\begin{align*}
& \tilde{d}(\resset , \check{\mu}) = \frac{1}{c} \sum_{j=1}^{c} \sup_{\patt \in \lang } \bigl\{ \hsqual{\patt}{\dataset^\star_j}{\check{\mu}} \bigr\}  \\
& \hat{r} = \tilde{d}(\resset , \check{\mu})  +
\sqrt{ \frac{ \unionbdelta }{ 2cm } } \\
& \hat{d} = \hat{r} + \sqrt{ \pars{\frac{2 \nu_T \unionbdelta }{ m }}^2 + \frac{ 2 \hat{r} \unionbdelta }{ m }} + \frac{2 \nu_T \unionbdelta }{ m } \\
& \varepsilon \doteq \hat{d} + \sqrt{\frac{2 \unionbdelta
    \left( \nu + 2\hat{d} \right)}{m}}
        + \frac{ \unionbdelta}{3m} \numberthis \label{eq:epsdev} . 
\end{align*}
With probability at least $1-\delta$ over the choice of $\dataset$ and $\resset$ it holds
$\sup_{\patt \in \lang^\star} \brpars{ \squalp  } \leq \varepsilon $.
\end{restatable}
We use the following simple upper bound for both $\nu_T$ and $\nu$:
$\nu_T = \nu = \sup_{|x - \mu(\dataset)| \leq \varepsilon_T} x(1-x) $, 
and note that while slightly more refined bounds are possible, we omit them to improve readability. 

Note that Theorem~\ref{thm:bound_dev} does not require the knowledge of the unknown parameter $\mu$, but only of an upper bound $\hat{\mu}$ and of a lower bound $\check{\mu}$.
Theorem~\ref{thm:bound_dev} allows to implement the procedure \texttt{boundStatistic} for the unconditional setting: 
\texttt{boundStatistic} 
computes $\varepsilon$ as in Eq.~\eqref{eq:epsdev}. Analogously to the conditional case, evaluating $\varepsilon$ requires to compute $\sup_{\patt \in \lang } \bigl\{ \hsqual{\patt}{\dataset^\star_j}{\check{\mu}} \bigr\}$, which is costly, but only needs to be performed on $c$ resampled datasets, and $c$ is in practice small (as we show in Section~\ref{sec:experiments}).

By combining the results in this section, we prove the guarantees provided by \algname\ for the task of finding significant patterns when unconditional testing is used. In particular, the following Corollary proves that,
when \texttt{boundStatistic} returns the value $\varepsilon$ defined in Eq.~\ref{eq:epsdev}, then the output set $O$ of significant patterns returned by \algnameu\ has FWER bounded by the user-defined parameter $\delta$.

\begin{corollary}
\label{cor:alg_uncond}
Fix $\delta \in (0,1)$ and $c \geq 1$. 
Let $O$ be the output of \algname\ with input parameters $\lang$, $\dataset$, $c$, $\delta$, and let 
$\varepsilon$ as in Eq.~\eqref{eq:epsdev}. 
Then, the set $O$ has FWER $\leq \delta$ under the unconditional null distribution. 
\end{corollary}

\textit{Power analysis.} The following result provides guarantees on the \emph{power} of \algnameu, the version of \algname\ that uses unconditional testing, building on the probabilistic upper bound to $\tilde{d}(\resset , \check{\mu})$ 
\ifextversion
proved in \Cref{sec:app_power}, 
\else
proved in Appendix~A.5, 
\fi
and on concentration bounds based on the pseudodimension of the language of subgroups.

\begin{theorem}
\label{thm:power_unconditional}
Fix $\delta \in (0,1)$ and $c , z \geq 1$. 
Let $O$ be the output of \algname\ with input parameters $\lang$, $\dataset$, $c$, and $\delta$,
where $\lang$ is the language of subgroups composed by conjunctions with at most $z$ conditions over $d$ continuous features. 
Define 
$\hat{\omega} = \hat{\mu} (1 - \check{\mu}) \sup_{\patt \in \lang} \freqp $, 
and let $\varepsilon$ be defined as in \Cref{thm:bound_dev}, where $\hat{r}$ is replaced by
\begin{align*}
& \hat{r} = \sqrt{ \frac{ 2 \hat{\omega} z \ln(\frac{e^3 m^{2} d}{4 z^3}) }{ m } } + \sqrt{ \frac{ \ln(\frac{3}{\delta}) }{ 2cm } } + \frac{ z \ln(\frac{e^3 m^{2} d}{4 z^3}) }{ 3m } +
\sqrt{ \frac{ \unionbdelta }{ 2cm } } . 
\end{align*}
Then, define $\tfreq = \E_\dataset[\freqp]$. With probability at least $1 - \delta$, $O$ contains \emph{all} patterns $\patt$ with quality $\tqual$ satisfying
\begin{align*}
\tqual \geq \varepsilon + 2 \varepsilon_T  \Biggl( \mathsf{f}_{\patt} + \sqrt{ \frac{ z \ln \bigl( \frac{2ed}{z} \bigr) + \ln \bigl( \frac{3}{\delta} \bigr) }{ 2m } } \Biggr) + \sqrt{ \frac{ z \ln \bigl( \frac{2ed}{z} \bigr) + \ln \bigl( \frac{3}{\delta} \bigr) }{ 2m } } .
\end{align*}
\end{theorem}

\section{Experiments}
\label{sec:experiments}

This section presents the results of our experiments.
The goal of our experimental evaluation is to assess \algname's capabilities of discovering significant patterns with high statistical power, analyzing efficiently large real-world datasets with few-shot resampling, in both conditional (\algnamec) and unconditional (\algnameu) settings.

\textit{Pattern Language.}
In our experimental evaluation we focus on the problem of discovering significant \emph{subgroups} from large real-world datasets with mixed feature types (both categorical and continuous). 
The language $\lang$ is composed of conjunctions of up to $z$ conditions on the features of the data \cite{atzmueller2015subgroup}, where $z$ is a fixed parameter (see below).
These conditions are either equalities (for categorical features), inequalities, or intervals (on continuous features). 

\textit{Datasets.}
We tested \algname\ on $12$ standard benchmarks and real-world datasets to evaluate subgroup discovery algorithms from UCI\footnote{https://archive.ics.uci.edu/}. 
The statistics of the datasets are described in \Cref{tab:datasets}.
These datasets 
cover 
a wide range of sizes, dimensionalities, and application domains. 
The column $z$ of \Cref{tab:datasets} reports the maximum number of conjunction terms for the subgroups in the language $\lang$ for each dataset.

\textit{Implementation of \algname.} 
We implemented \algname\ in \texttt{Python}. 
The code and the scripts to reproduce all experiments are available online\footnote{https://github.com/VandinLab/FSR}. 
To mine subgroups, we make use of a fast depth-first enumeration algorithm included in the library \texttt{pysubgroup}\footnote{https://github.com/flemmerich/pysubgroup}~\cite{lemmerich2019pysubgroup}.

\textit{Baselines.}
Since our algorithm \algname\ is the first algorithm that can be used for mining significant patterns with both conditional and unconditional testing, we consider different baselines for conditional testing and for unconditional testing.

For conditional testing, we compare \algnamec\ with a variant of TopKWY~\cite{pellegrina2020efficient}, the state-of-the-art method for significant pattern mining with conditional testing, based on the WY permutation testing procedure~\cite{westfall1993resampling}. The original implementation of TopKWY is only tailored to identify significant itemsets and subgraphs, and does not support subgroups from categorical and continuous features; however, we note that it is fairly simple to adapt its strategy to such case. In particular, we extended TopKWY  to identify significant subgroups by estimating the distribution of the supremum deviation using permuted datasets (instead of the $p$-values as done in the original TopKWY implementation~\cite{pellegrina2020efficient}). That is, our variant of TopKWY considers the same statistic of \algnamec\ (i.e., the supremum deviation), but instead of generating resamples and taking the average of supremum deviations as done by \algnamec, 
it efficiently computes its $\delta$-quantiles considering permutations of the labels. In addition, in our variant of TopKWY we use \texttt{pysubgroup} to mine subgroups. Note that since our variant of TopKWY and \algnamec\ share the same, equally optimized, procedure to explore the search space of the language $\lang$, all comparisons in terms of running times are fair.  For TopKWY we use $10^3$ permutations, a good trade-off in terms of running time and accuracy for estimating the $\delta$-quantile for typical values of $\delta$ (e.g., $0.05$)~\cite{llinares2015fast}. 

Regarding unconditional testing, we note that \algnameu\ is the first method to discover significant patterns in a (fully) unconditional setting.
Therefore, a baseline that may seem reasonable to control the FWER is the standard Bonferroni correction: 
each pattern $\patt$ is flagged as significant if the probability (under the null hypothesis) of observing a quality greater or equal than the one measured in the data is at most $\delta/|\lang|$, where $|\lang|$ is the size of the language $\lang$ (see also \Cref{sec:intro}). 
Note, however, that this simple method is not useful for subgroups as $|\lang|$ is infinite (e.g., as the number of inequalities over a continuous feature is unbounded)\footnote{We remark that, while there exist other correction procedures more powerful than Bonferroni (e.g., the Holm procedure), they all require to fix a finite set of hypothesis, therefore do not apply to our setting.}. 
Therefore, we compare \algnameu\ with a novel non-trivial baseline, that we call \algnameuub, that resolves this issue. 
We describe \algnameuub\ at high level and defer additional details to 
\ifextversion
\Cref{sec:baselines}.
\else
Appendix~A.6.
\fi
For \algnameuub, we use part of the concentration bounds developed for \algnameu\ to upper bound the supremum deviation in terms of its expectation (taken w.r.t. the resamples but conditionally on the transactions, 
\ifextversion
see Theorem~\ref{thm:bound_dev} and \Cref{sec:baselines}), 
\else
see Theorem~\ref{thm:bound_dev} and Appendix~A.6, 
\fi
but instead of estimating the expected supremum deviation 
with $c$ resamples (as done by \algnameu),
we compute an upper bound to it via an union bound over the (finite) number of distinct subgroups that are \emph{observed} in the data in at least one transaction, i.e., we consider the finite \emph{projection} of $\lang$ on $\dataset$, instead of \emph{all} $|\lang|$ possible subgroups. 
Note that such approach bounds the FWER while being much less conservative than Bonferroni, since it corrects for the size of the projection instead of the size of $\lang$. 
By comparing \algnameu\ to \algnameuub\ we directly evaluate the advantage of computing bounds from resampled datasets that consider dependencies among patterns, as done by \algnameu.

\textit{Experimental setup.} 
All the experiments were run on a machine with 2.30 GHz Intel Xeon CPU, 512 GB of RAM, on Ubuntu 20.04.
In all experiments we use $\delta = 0.05$ (i.e., we control the FWER below $0.05$). 
We repeated all experiments $10$ times, and report averages $\pm $ stds over the $10$ repetitions.

\begin{table}
  \caption{Statistics of the datasets considered in our experiments.  
  $m$ is the number of transactions, $d$ is the number of features (categorical/continuous), $\mu(\dataset)$ is the fraction of transactions with target equal to $1$, $z$ is the maximum number of conjunction terms in the language $\lang$. 
  }
\label{tab:datasets}
\center
  \begin{tabular}{lrrrrrr}
    \toprule
    $\dataset$              & $m$      & $d$ & $\mu(\dataset)$  & $z$     \\
    \midrule
    abalone     & 4177 & 1/7 & 0.663 & 5 \\
	adult     & 32561 & 8/6  & 0.241 & 5 \\
	bank     & 41188 &  10/10 & 0.113 & 3 \\
	brain-cancer     & 862 & 22/1  & 0.421 & 5 \\
	cancer-rna-seq     & 801 & 0/20531  & 0.375 & 2 \\
	covtype     & 581012 & 0/54  & 0.365 & 3 \\
	gisette     & 7000 & 0/5000 & 0.500 &  2 \\
	HIGGS     & 11000000 & 0/28 & 0.529 &  3 \\
	kdd-cup & 95370 & 73/405 & 0.050 &  2 \\
	mushroom     & 8124 & 22/0 & 0.482 & 5  \\
	SUSY     & 5000000 & 0/18 & 0.457 &  3 \\
	theorem-prover     & 3059 & 0/51 & 0.420 & 3  \\
  \bottomrule
\end{tabular}
\end{table}

\begin{figure*}[ht]
\begin{subfigure}{.247\textwidth}
  \centering
  \includegraphics[width=\textwidth]{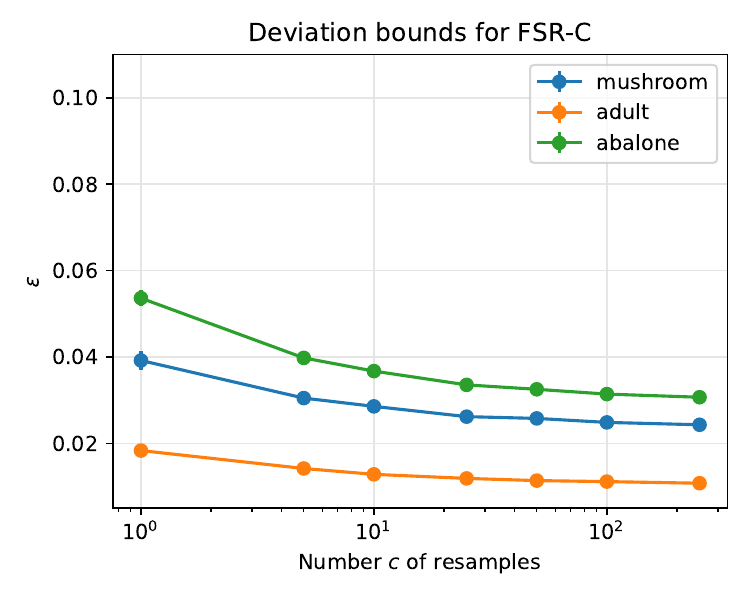}
  \caption{}
\end{subfigure}
\begin{subfigure}{.247\textwidth}
  \centering
  \includegraphics[width=\textwidth]{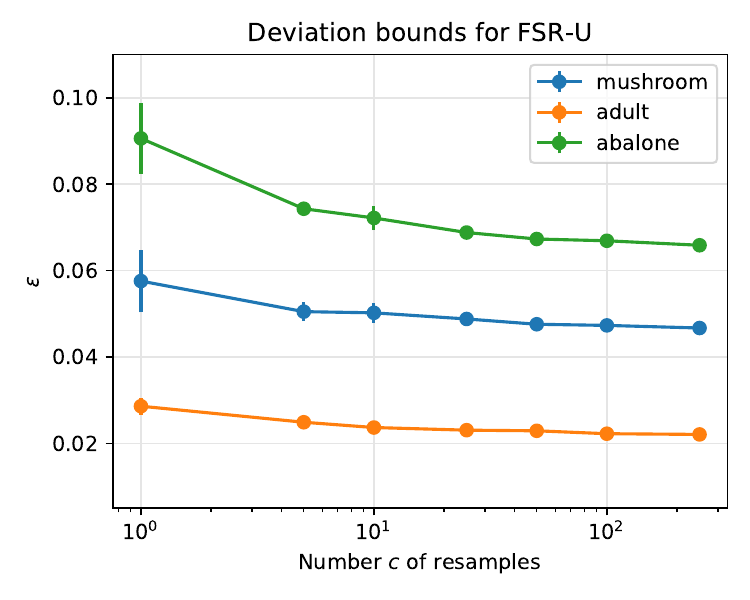}
  \caption{}
\end{subfigure}
\begin{subfigure}{.247\textwidth}
  \centering
  \includegraphics[width=\textwidth]{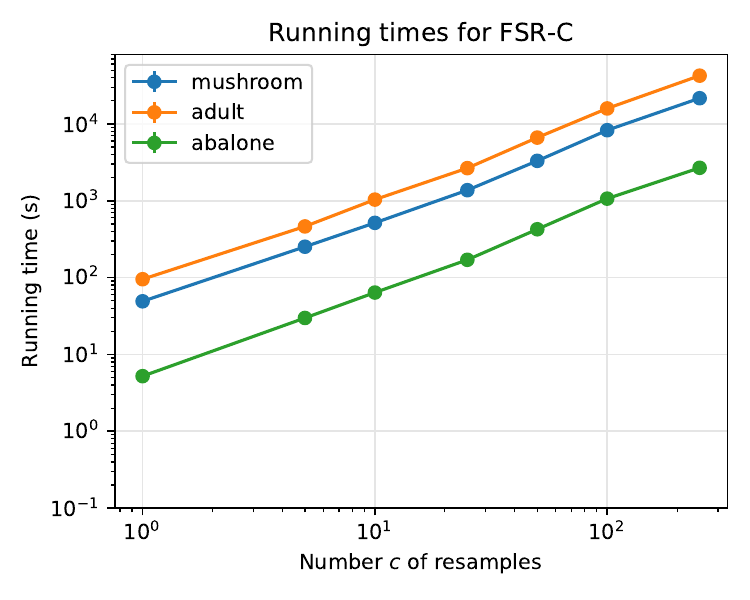}
  \caption{}
\end{subfigure}
\begin{subfigure}{.247\textwidth}
  \centering
  \includegraphics[width=\textwidth]{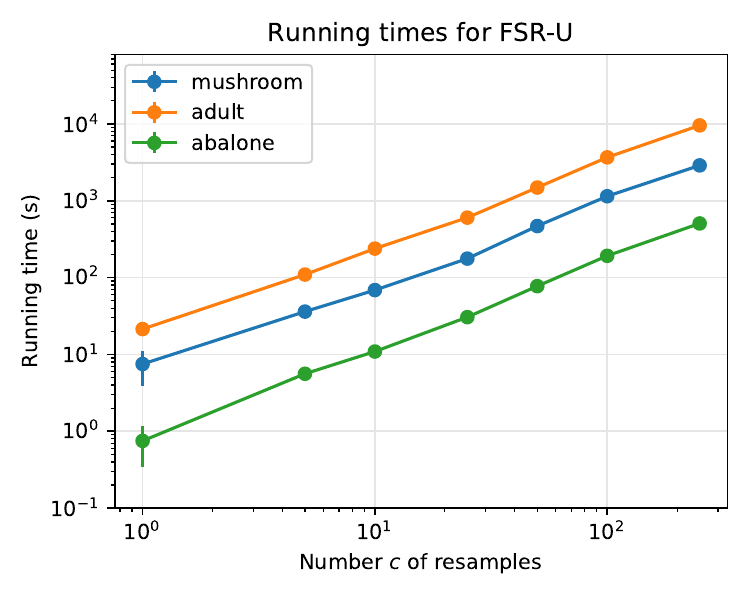}
  \caption{}
\end{subfigure}
\caption{ 
Effect of the number of resamples $c$ on the deviation bounds (a)-(b) and running times (c)-(d) for \algname\ algorithms. 
}
\label{fig:parameters}
\end{figure*}

\begin{figure*}[ht]
\centering
\begin{subfigure}{.15\textwidth}
  \centering
  \includegraphics[width=\textwidth]{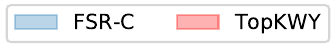}
\end{subfigure} \\
\begin{subfigure}{.247\textwidth}
  \centering
  \includegraphics[width=\textwidth]{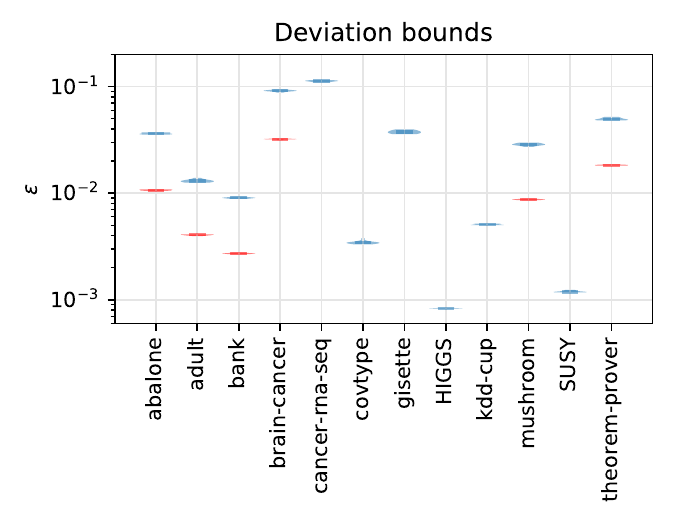}
  \caption{}
\end{subfigure}
\begin{subfigure}{.247\textwidth}
  \centering
  \includegraphics[width=\textwidth]{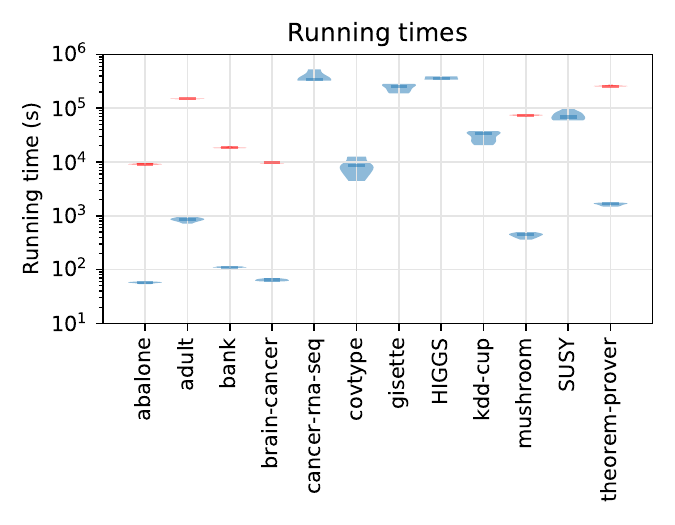}
  \caption{}
\end{subfigure}
\begin{subfigure}{.247\textwidth}
  \centering
  \includegraphics[width=\textwidth]{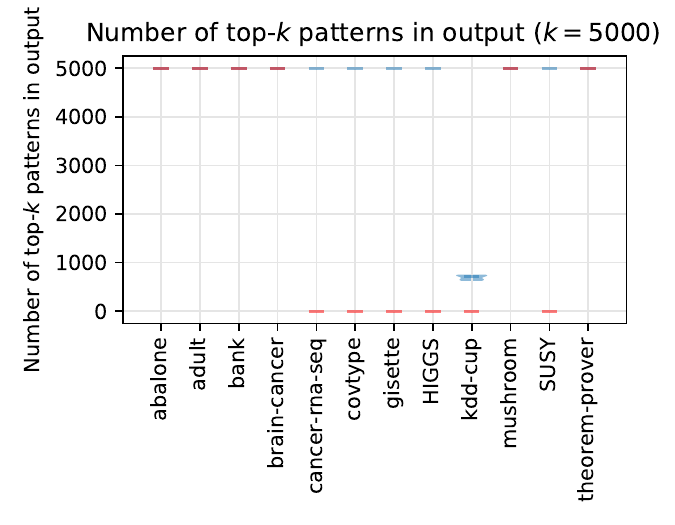}
  \caption{}
\end{subfigure}
\begin{subfigure}{.247\textwidth}
  \centering
  \includegraphics[width=\textwidth]{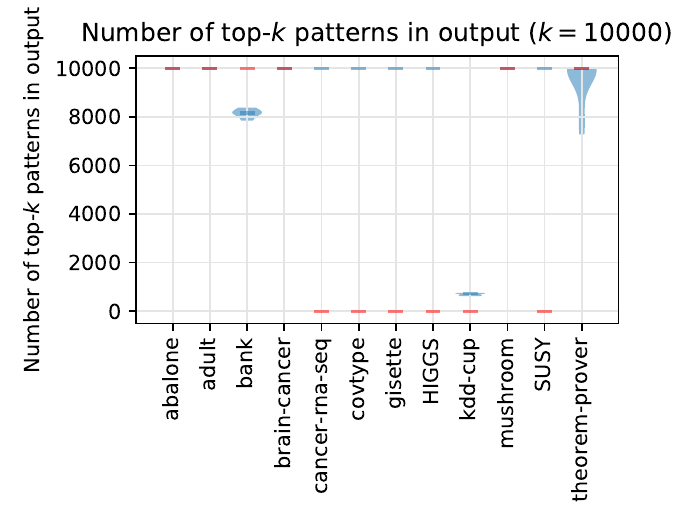}
  \caption{}
\end{subfigure}
\caption{ 
Comparison of \algnamec\ with TopKWY in terms of deviation bounds (a), running times (b), and number of results (c)-(d). 
}
\label{fig:condresults}
\end{figure*}

\begin{figure*}[ht]
\centering
\begin{subfigure}{.15\textwidth}
  \centering
  \includegraphics[width=\textwidth]{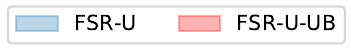}
\end{subfigure} \\
\begin{subfigure}{.247\textwidth}
  \centering
  \includegraphics[width=\textwidth]{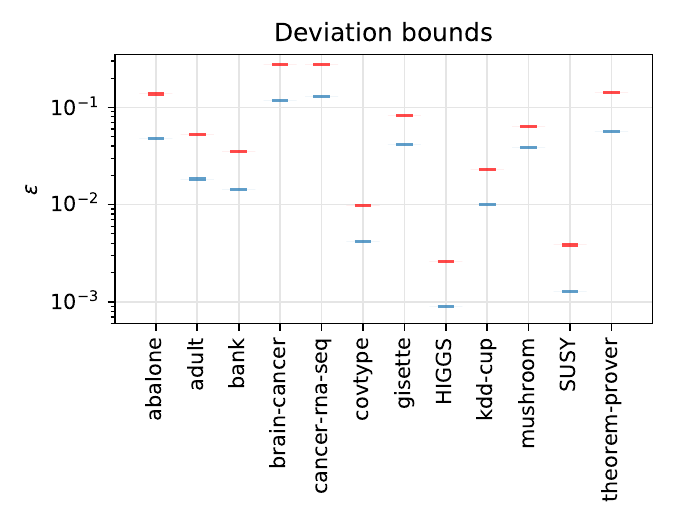}
  \caption{}
\end{subfigure}
\begin{subfigure}{.247\textwidth}
  \centering
  \includegraphics[width=\textwidth]{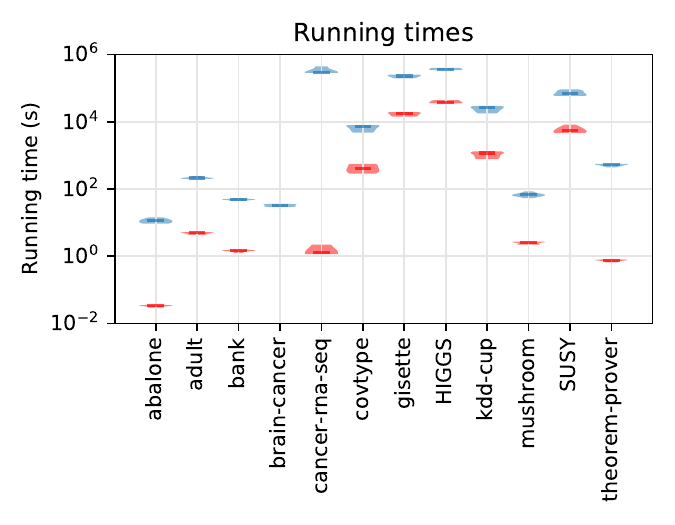}
  \caption{}
\end{subfigure}
\begin{subfigure}{.247\textwidth}
  \centering
  \includegraphics[width=\textwidth]{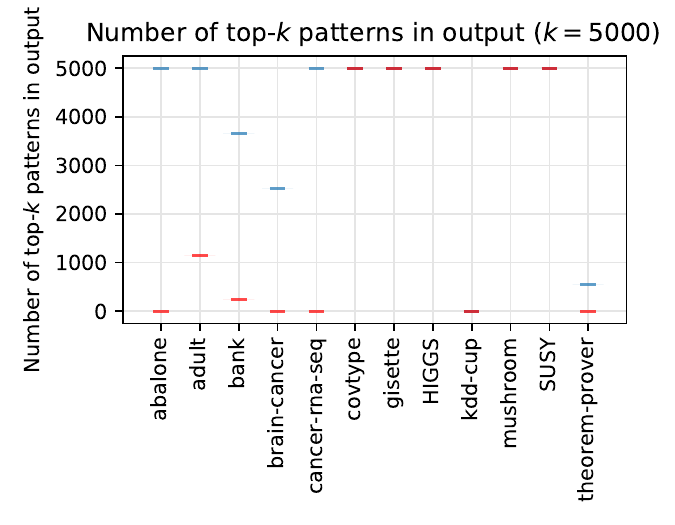}
  \caption{}
\end{subfigure}
\begin{subfigure}{.247\textwidth}
  \centering
  \includegraphics[width=\textwidth]{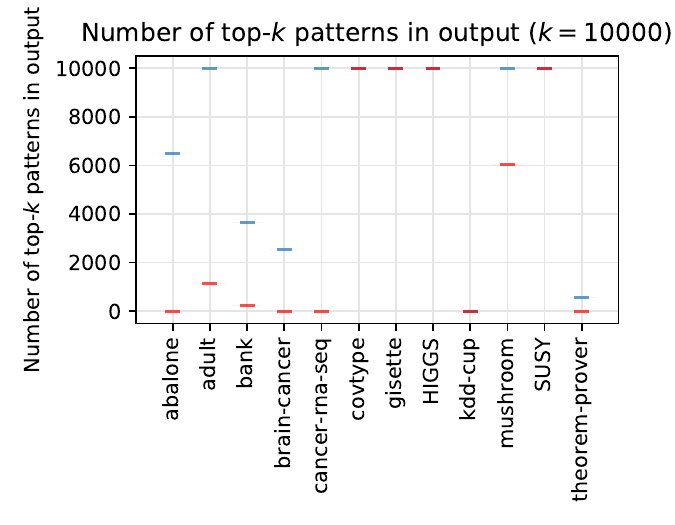}
  \caption{}
\end{subfigure}
\caption{ Comparison of \algnameu\ with the baseline \algnameuub\ 
\ifextversion
(\Cref{sec:baselines}) 
\else
(Appendix~A.6) 
\fi
in terms of deviation bounds (a), running times (b), and number of results (c)-(d).
}
\label{fig:uncondresults}
\end{figure*}

\textit{Impact of parameters on \algname.}
In the first set of experiments we evaluate the effect of the number of resamples $c$ on the deviation bound $\varepsilon$ computed by \algname\ and its running time. 
We consider both \algnamec\ and \algnameu, respectively designed to compute significant patterns with conditional and unconditional testing. 
To ease the presentation, for this first experiment we focus on $3$ of the datasets we considered; the results for the other datasets are very similar. 
In Figure~\ref{fig:parameters}-(a) we show the deviation bounds computed by \algnamec\ for different values of $c$,
while Figure~\ref{fig:parameters}-(b) is analogous for \algnameu. 
From these plots we clearly conclude that using $c=10$ resamples is sufficient to obtain a small deviation bound, and that using more than $10$ resamples is marginally beneficial as all curves flatten. 
Remarkably, for all datasets (and in particular for the larger dataset \texttt{adult}), and for both methods, even using \emph{one} resample is enough to compute a meaningful deviation bound. 
This is in striking contrast with state-of-the-art methods based on permutation testing, that instead require a number of $10^3$-$10^4$ permutations of the target labels to properly estimate the $\delta$-quantile. 
Furthermore, we observe that the deviation bounds computed by \algnamec, in the conditional setting, are smaller than the bounds computed by \algnameu, for the unconditional scenario;
this confirms the fact that the assumptions made on the process generating the data have a sensible effect on this aspect, a consequence of properly taking into account the \emph{uncertainty} of the collected data. 

Figures~\ref{fig:parameters}-(c)-(d)  show the running time of the two methods \algnamec\ and \algnameu\ as functions of $c$.
Not surprisingly, we clearly observe that the running time increases linearly with the number of resamples $c$ for both methods.
Therefore, we expect that processing a small number of resampled datasets is extremely advantageous in terms of running time. 
We also observe that \algnameu\ is faster than \algnamec\ for all values of $c$. 
This is due to the fact that \algnamec, by computing a smaller deviation bound, explores a wider portion of the pattern language $\lang$. 
In any case, using at most $10$ resamples is always feasible (since both algorithms terminate after at most $17$ minutes for these datasets).  
From these observation, we fix $c$ to $10$ for \algnamec\ and \algnameu\ in all our experiments.

\textit{Evaluation of \algnamec.}
In this experiment we report the performance of \algnamec\ to identify significant patterns with conditional testing, comparing it with a variant of TopKWY, the state-of-the-art method to identify significant patterns with permutation testing. 
In Figure~\ref{fig:condresults} we show the deviation bounds (a), the running times (b), and the number of reported results (c)-(d) for both algorithms.
From Figure~\ref{fig:condresults}-(a), we observe that the deviation bounds computed by \algnamec\ are larger than the ones computed by TopKWY. 
This is not surprising, as the deviation bound returned by TopKWY, that is based on the WY procedure, is a very accurate estimate of the optimal corrected threshold to bound the FWER (since the WY estimator converges asymptotically to it~\cite{meinshausen2011asymptotic}); 
on the other hand, \algnamec\ computes an \emph{upper bound} to such quantity with stronger probabilistic guarantees 
(i.e., the bound does not only converge \emph{asymptotically} and \emph{in expectation}, but rather holds in finite samples with high probability).
While the deviation bounds computed by TopKWY are smaller than \algnamec, from Figure~\ref{fig:condresults}-(b) we observe a significant gap in terms of running time between the two methods. 
In fact, TopKWY requires two orders of magnitude more time than \algnamec\ to compute its deviation bound; 
this is mainly due to the fact that TopKWY has to process two orders of magnitude more permutations of the target label than \algnamec. 
Note that using $10^4$ permutations (instead of $10^3$, to have a more accurate estimation of the $\delta$-quantile, as often done in practice) in TopKWY would result in an even more substantial gap.

\balance
We now compare the two methods in terms of number of reported significant patterns. 
To do so, we follow a typical subgroup discovery analysis, based on the task of identifying the $k$ most \qt{interesting} subgroups in terms of quality: 
first, we mine the set of top-$k$ subgroups with highest quality; then, we count the number of them that are flagged as significant by the two methods. 
We use $k \in \{ 5 \cdot 10^3 , 10^4\}$. 
Figures~\ref{fig:condresults}-(c)-(d) show the number of patterns that are reported in output by both methods.
We observe that, for $k=5 \cdot 10^3$, for $6$ of the $12$ datasets we considered, both algorithms output all the top-$k$ patterns, i.e., they have the same output; for such datasets TopKWY reports all top-$k$ patterns also for $k=10^4$, while  
\algnamec\ outputs the same set of results for all but two datasets, for which it reports more than $70\%$ of them. 
For the remaining $6$ datasets, which consist of the 4 largest datasets (in terms of the number $m$ of transactions) and the 2 datasets with highest number  of features, TopKWY could not complete in reasonable time (i.e., we stopped it after $10$ days), while \algnamec\ finished the analysis while returning a large number of significant results (i.e., either all the top-$k$ patterns or more than 600 of them for the \texttt{kdd-cup} dataset). 
For the most challenging datasets (\texttt{HIGGS}, with $11$ millions transactions, and \texttt{cancer-rna-seq}, with $> 10^4$ continuous features), \algnamec\ concludes after 5 days, while TopKWY would require (approximately) $1.5$ years of computation.  
Therefore, in such cases \algnamec\ \emph{enables} the analysis of these challenging instances while identifying many significant patterns. 
These results confirm that \algnamec, while providing a slightly more conservative upper bound to the supremum deviation, is still capable of discovering the same (or almost the same) most significant patterns, while being more than two orders of magnitude faster, and reporting many patterns as significant for challenging instances out of reach for the state-of-the-art.
We conclude that \algnamec\ provides an excellent trade-off between the number of patterns identified as significant and the computational requirement of the analysis.

\textit{Evaluation of \algnameu.}
We now evaluate the performance of \algnameu\ to discover significant patterns with unconditional testing.
We show the results in Figure~\ref{fig:uncondresults}. We compare
\algnameu\ with the baseline \algnameuub\ (described above)
in terms of deviation bounds (a), running time (b), and number of results (c)-(d).
From Figure~\ref{fig:uncondresults}-(a) we clearly observe that \algnameu\ computes deviation bounds that are always smaller than \algnameuub.
We conclude that processing the resamples of the target label provides much more accurate deviation bounds w.r.t. more standard techniques (i.e., a Bonferroni correction). 
This is a consequence of taking into account the dependencies among patterns when upper bounding the expected supremum deviation. 
In terms of running time, \algnameu\ always conclude in reasonable time (similarly to \algnamec), while \algnameuub\ is faster (since it does not consider any resamples). 
On the other hand, in many cases the baseline \algnameuub\ terminates quickly but without reporting anything in output: Figures~\ref{fig:uncondresults}-(c)-(d) show that for $5$ datasets it does not report significant patterns, while for the other datasets it outputs a significantly smaller amount (e.g., for \texttt{adult}, \algnameu\ finds almost an order of magnitude more results than \algnameuub). 
This experiment shows that \algnameu\ is a practical and powerful method to discover significant patterns with unconditional testing, and that it significantly improves over more standard techniques such as Bonferroni correction.

\textit{Application to Neural Network interpretation.}
In this final experiment we evaluate a practical application of \algname\ to the task of Neural Network interpretation \cite{fischer21b}.
More precisely, we consider the MNIST dataset \cite{mnistdata} and train a Convolutional Neural Network (CNN), with the goal of identifying correlations between the activation values of neurons with the predicted target. 
To do so, we evaluate the association of the activation values of neurons in a convolutional filter with a binary target, composed by drawings of digits composed by straight lines only ($1$ and $7$), versus the other digits. 
While we defer most details of this experiment to 
\ifextversion
\Cref{sec:mnistexperiment}, 
\else
Appendix~A.7, 
\fi
we observed that \algname\ successfully identifies interpretable activation patterns of several neurons, while requiring orders of magnitude less time than previous methods (as discussed previously). 

\section{Conclusions}
We presented \algname, a novel algorithm to identify statistically significant patterns with rigorous bounds on the FWER. \algname\ uses a few-shot resampling strategy, which leads to an efficient and practical approach that can be used for both conditional and unconditional testing. Our experimental evaluation shows that  \algname\ is an effective and accurate method for significant patterns discovery, while significantly reducing the computational cost of state-of-the-art multiple comparisons procedures, such as permutation testing, that hardly scale the analysis to complex languages, such as subgroups, and large datasets.
While the experiments presented in this work are focused on subgroups, we expect the relative improvements obtained by \algname\ to directly transfer to other pattern types, given the generality of our framework and of the design of our resampling procedures, which are shared by all types of patterns.

\begin{acks}
This work was supported by the \qt{National Center for HPC, Big
Data, and Quantum Computing}, project CN00000013, 
and by the 
PRIN Project n. 2022TS4Y3N - EXPAND: scalable algorithms for EXPloratory Analyses of heterogeneous and dynamic Networked Data,
funded by the
Italian Ministry of University and Research (MUR). 
\end{acks}

\newpage

\bibliographystyle{ACM-Reference-Format}
\bibliography{bibliography}


\begin{thebibliography}{46}


\ifx \showCODEN    \undefined \def \showCODEN     #1{\unskip}     \fi
\ifx \showDOI      \undefined \def \showDOI       #1{#1}\fi
\ifx \showISBNx    \undefined \def \showISBNx     #1{\unskip}     \fi
\ifx \showISBNxiii \undefined \def \showISBNxiii  #1{\unskip}     \fi
\ifx \showISSN     \undefined \def \showISSN      #1{\unskip}     \fi
\ifx \showLCCN     \undefined \def \showLCCN      #1{\unskip}     \fi
\ifx \shownote     \undefined \def \shownote      #1{#1}          \fi
\ifx \showarticletitle \undefined \def \showarticletitle #1{#1}   \fi
\ifx \showURL      \undefined \def \showURL       {\relax}        \fi
\providecommand\bibfield[2]{#2}
\providecommand\bibinfo[2]{#2}
\providecommand\natexlab[1]{#1}
\providecommand\showeprint[2][]{arXiv:#2}

\bibitem[\protect\citeauthoryear{Aggarwal, Li, Yu, and Jin}{Aggarwal
  et~al\mbox{.}}{2010}]%
        {aggarwal2010dense}
\bibfield{author}{\bibinfo{person}{Charu~C Aggarwal}, \bibinfo{person}{Yao Li},
  \bibinfo{person}{Philip~S Yu}, {and} \bibinfo{person}{Ruoming Jin}.}
  \bibinfo{year}{2010}\natexlab{}.
\newblock \showarticletitle{On dense pattern mining in graph streams}.
\newblock \bibinfo{journal}{\emph{Proceedings of the VLDB Endowment}}
  \bibinfo{volume}{3}, \bibinfo{number}{1-2} (\bibinfo{year}{2010}),
  \bibinfo{pages}{975--984}.
\newblock


\bibitem[\protect\citeauthoryear{Agrawal, Imieli{\'n}ski, and Swami}{Agrawal
  et~al\mbox{.}}{1993}]%
        {agrawal1993mining}
\bibfield{author}{\bibinfo{person}{Rakesh Agrawal}, \bibinfo{person}{Tomasz
  Imieli{\'n}ski}, {and} \bibinfo{person}{Arun Swami}.}
  \bibinfo{year}{1993}\natexlab{}.
\newblock \showarticletitle{Mining association rules between sets of items in
  large databases}. In \bibinfo{booktitle}{\emph{Proceedings of the 1993 ACM
  SIGMOD international conference on Management of data}}.
  \bibinfo{pages}{207--216}.
\newblock


\bibitem[\protect\citeauthoryear{Al~Hasan and Zaki}{Al~Hasan and Zaki}{2009}]%
        {al2009output}
\bibfield{author}{\bibinfo{person}{Mohammad Al~Hasan} {and}
  \bibinfo{person}{Mohammed~J Zaki}.} \bibinfo{year}{2009}\natexlab{}.
\newblock \showarticletitle{Output space sampling for graph patterns}.
\newblock \bibinfo{journal}{\emph{Proceedings of the VLDB Endowment}}
  \bibinfo{volume}{2}, \bibinfo{number}{1} (\bibinfo{year}{2009}),
  \bibinfo{pages}{730--741}.
\newblock


\bibitem[\protect\citeauthoryear{Atzmueller}{Atzmueller}{2015}]%
        {atzmueller2015subgroup}
\bibfield{author}{\bibinfo{person}{Martin Atzmueller}.}
  \bibinfo{year}{2015}\natexlab{}.
\newblock \showarticletitle{Subgroup discovery}.
\newblock \bibinfo{journal}{\emph{Wiley Interdisciplinary Reviews: Data Mining
  and Knowledge Discovery}} \bibinfo{volume}{5}, \bibinfo{number}{1}
  (\bibinfo{year}{2015}), \bibinfo{pages}{35--49}.
\newblock


\bibitem[\protect\citeauthoryear{Barnard}{Barnard}{1945}]%
        {barnard1945new}
\bibfield{author}{\bibinfo{person}{GA Barnard}.}
  \bibinfo{year}{1945}\natexlab{}.
\newblock \showarticletitle{A new test for 2$\times$ 2 tables.}
\newblock \bibinfo{journal}{\emph{Nature}} \bibinfo{volume}{156},
  \bibinfo{number}{3954} (\bibinfo{year}{1945}).
\newblock


\bibitem[\protect\citeauthoryear{Benjamini and Hochberg}{Benjamini and
  Hochberg}{1995}]%
        {benjamini1995controlling}
\bibfield{author}{\bibinfo{person}{Yoav Benjamini} {and} \bibinfo{person}{Yosef
  Hochberg}.} \bibinfo{year}{1995}\natexlab{}.
\newblock \showarticletitle{Controlling the false discovery rate: a practical
  and powerful approach to multiple testing}.
\newblock \bibinfo{journal}{\emph{Journal of the Royal statistical society:
  series B (Methodological)}} \bibinfo{volume}{57}, \bibinfo{number}{1}
  (\bibinfo{year}{1995}), \bibinfo{pages}{289--300}.
\newblock


\bibitem[\protect\citeauthoryear{Bonferroni}{Bonferroni}{1936}]%
        {bonferroni1936teoria}
\bibfield{author}{\bibinfo{person}{Carlo Bonferroni}.}
  \bibinfo{year}{1936}\natexlab{}.
\newblock \showarticletitle{Teoria statistica delle classi e calcolo delle
  probabilita}.
\newblock \bibinfo{journal}{\emph{Pubblicazioni del R Istituto Superiore di
  Scienze Economiche e Commericiali di Firenze}}  \bibinfo{volume}{8}
  (\bibinfo{year}{1936}), \bibinfo{pages}{3--62}.
\newblock


\bibitem[\protect\citeauthoryear{Boucheron, Lugosi, and Massart}{Boucheron
  et~al\mbox{.}}{2013}]%
        {boucheron2013concentration}
\bibfield{author}{\bibinfo{person}{St{\'e}phane Boucheron},
  \bibinfo{person}{G{\'a}bor Lugosi}, {and} \bibinfo{person}{Pascal Massart}.}
  \bibinfo{year}{2013}\natexlab{}.
\newblock \bibinfo{booktitle}{\emph{Concentration inequalities: A nonasymptotic
  theory of independence}}.
\newblock \bibinfo{publisher}{Oxford university press}.
\newblock


\bibitem[\protect\citeauthoryear{Cao, Yan, Madden, Rundensteiner, and
  Gopalsamy}{Cao et~al\mbox{.}}{2019}]%
        {cao2019efficient}
\bibfield{author}{\bibinfo{person}{Lei Cao}, \bibinfo{person}{Yizhou Yan},
  \bibinfo{person}{Samuel Madden}, \bibinfo{person}{Elke~A Rundensteiner},
  {and} \bibinfo{person}{Mathan Gopalsamy}.} \bibinfo{year}{2019}\natexlab{}.
\newblock \showarticletitle{Efficient discovery of sequence outlier patterns}.
\newblock \bibinfo{journal}{\emph{Proceedings of the VLDB Endowment}}
  \bibinfo{volume}{12}, \bibinfo{number}{8} (\bibinfo{year}{2019}),
  \bibinfo{pages}{920--932}.
\newblock


\bibitem[\protect\citeauthoryear{Ceccarello and Gamper}{Ceccarello and
  Gamper}{2022}]%
        {ceccarello2022fast}
\bibfield{author}{\bibinfo{person}{Matteo Ceccarello} {and}
  \bibinfo{person}{Johann Gamper}.} \bibinfo{year}{2022}\natexlab{}.
\newblock \showarticletitle{Fast and Scalable Mining of Time Series Motifs with
  Probabilistic Guarantees}.
\newblock \bibinfo{journal}{\emph{Proceedings of the VLDB Endowment}}
  \bibinfo{volume}{15}, \bibinfo{number}{13} (\bibinfo{year}{2022}),
  \bibinfo{pages}{3841--3853}.
\newblock


\bibitem[\protect\citeauthoryear{Chen, Lin, Fredrikson, Christodorescu, Yan,
  and Han}{Chen et~al\mbox{.}}{2009}]%
        {chen2009mining}
\bibfield{author}{\bibinfo{person}{Chen Chen}, \bibinfo{person}{Cindy~X Lin},
  \bibinfo{person}{Matt Fredrikson}, \bibinfo{person}{Mihai Christodorescu},
  \bibinfo{person}{Xifeng Yan}, {and} \bibinfo{person}{Jiawei Han}.}
  \bibinfo{year}{2009}\natexlab{}.
\newblock \showarticletitle{Mining graph patterns efficiently via randomized
  summaries}.
\newblock \bibinfo{journal}{\emph{Proceedings of the VLDB Endowment}}
  \bibinfo{volume}{2}, \bibinfo{number}{1} (\bibinfo{year}{2009}),
  \bibinfo{pages}{742--753}.
\newblock


\bibitem[\protect\citeauthoryear{Dalleiger and Vreeken}{Dalleiger and
  Vreeken}{2022}]%
        {dalleiger2022discovering}
\bibfield{author}{\bibinfo{person}{Sebastian Dalleiger} {and}
  \bibinfo{person}{Jilles Vreeken}.} \bibinfo{year}{2022}\natexlab{}.
\newblock \showarticletitle{Discovering significant patterns under sequential
  false discovery control}. In \bibinfo{booktitle}{\emph{Proceedings of the
  28th ACM SIGKDD Conference on Knowledge Discovery and Data Mining}}.
\newblock


\bibitem[\protect\citeauthoryear{Dubhashi and Panconesi}{Dubhashi and
  Panconesi}{2009}]%
        {dubhashi2009concentration}
\bibfield{author}{\bibinfo{person}{Devdatt~P Dubhashi} {and}
  \bibinfo{person}{Alessandro Panconesi}.} \bibinfo{year}{2009}\natexlab{}.
\newblock \bibinfo{booktitle}{\emph{Concentration of measure for the analysis
  of randomized algorithms}}.
\newblock \bibinfo{publisher}{Cambridge University Press}.
\newblock


\bibitem[\protect\citeauthoryear{Fischer, Olah, and Vreeken}{Fischer
  et~al\mbox{.}}{2021}]%
        {fischer21b}
\bibfield{author}{\bibinfo{person}{Jonas Fischer}, \bibinfo{person}{Anna Olah},
  {and} \bibinfo{person}{Jilles Vreeken}.} \bibinfo{year}{2021}\natexlab{}.
\newblock \showarticletitle{What’s in the Box? Exploring the Inner Life of
  Neural Networks with Robust Rules}. In \bibinfo{booktitle}{\emph{Proceedings
  of the 38th International Conference on Machine Learning}}
  \emph{(\bibinfo{series}{Proceedings of Machine Learning Research})},
  Vol.~\bibinfo{volume}{139}. \bibinfo{publisher}{PMLR},
  \bibinfo{pages}{3352--3362}.
\newblock


\bibitem[\protect\citeauthoryear{Fisher}{Fisher}{1922}]%
        {fisher1922interpretation}
\bibfield{author}{\bibinfo{person}{Ronald~A Fisher}.}
  \bibinfo{year}{1922}\natexlab{}.
\newblock \showarticletitle{On the interpretation of $\chi$ 2 from contingency
  tables, and the calculation of P}.
\newblock \bibinfo{journal}{\emph{Journal of the royal statistical society}}
  \bibinfo{volume}{85}, \bibinfo{number}{1} (\bibinfo{year}{1922}).
\newblock


\bibitem[\protect\citeauthoryear{H{\"a}m{\"a}l{\"a}inen and
  Webb}{H{\"a}m{\"a}l{\"a}inen and Webb}{2019}]%
        {hamalainen2019tutorial}
\bibfield{author}{\bibinfo{person}{Wilhelmiina H{\"a}m{\"a}l{\"a}inen} {and}
  \bibinfo{person}{Geoffrey~I Webb}.} \bibinfo{year}{2019}\natexlab{}.
\newblock \showarticletitle{A tutorial on statistically sound pattern
  discovery}.
\newblock \bibinfo{journal}{\emph{Data Mining and Knowledge Discovery}}
  \bibinfo{volume}{33} (\bibinfo{year}{2019}), \bibinfo{pages}{325--377}.
\newblock


\bibitem[\protect\citeauthoryear{Han, Cheng, Xin, and Yan}{Han
  et~al\mbox{.}}{2007}]%
        {han2007frequent}
\bibfield{author}{\bibinfo{person}{Jiawei Han}, \bibinfo{person}{Hong Cheng},
  \bibinfo{person}{Dong Xin}, {and} \bibinfo{person}{Xifeng Yan}.}
  \bibinfo{year}{2007}\natexlab{}.
\newblock \showarticletitle{Frequent pattern mining: current status and future
  directions}.
\newblock \bibinfo{journal}{\emph{Data mining and knowledge discovery}}
  \bibinfo{volume}{15}, \bibinfo{number}{1} (\bibinfo{year}{2007}),
  \bibinfo{pages}{55--86}.
\newblock


\bibitem[\protect\citeauthoryear{Ho, Pedersen, et~al\mbox{.}}{Ho
  et~al\mbox{.}}{2022}]%
        {ho2022efficient}
\bibfield{author}{\bibinfo{person}{Nguyen Thi~Thao Ho},
  \bibinfo{person}{Torben~Bach Pedersen}, {et~al\mbox{.}}}
  \bibinfo{year}{2022}\natexlab{}.
\newblock \showarticletitle{Efficient temporal pattern mining in big time
  series using mutual information}.
\newblock \bibinfo{journal}{\emph{Proceedings of the VLDB Endowment}}
  \bibinfo{volume}{15}, \bibinfo{number}{3} (\bibinfo{year}{2022}),
  \bibinfo{pages}{673--685}.
\newblock


\bibitem[\protect\citeauthoryear{Holm}{Holm}{1979}]%
        {holm1979simple}
\bibfield{author}{\bibinfo{person}{Sture Holm}.}
  \bibinfo{year}{1979}\natexlab{}.
\newblock \showarticletitle{A simple sequentially rejective multiple test
  procedure}.
\newblock \bibinfo{journal}{\emph{Scandinavian journal of statistics}}
  (\bibinfo{year}{1979}), \bibinfo{pages}{65--70}.
\newblock


\bibitem[\protect\citeauthoryear{Jogdeo and Samuels}{Jogdeo and
  Samuels}{1968}]%
        {jogdeo1968monotone}
\bibfield{author}{\bibinfo{person}{Kumar Jogdeo} {and}
  \bibinfo{person}{Stephen~M Samuels}.} \bibinfo{year}{1968}\natexlab{}.
\newblock \showarticletitle{Monotone convergence of binomial probabilities and
  a generalization of Ramanujan's equation}.
\newblock \bibinfo{journal}{\emph{The Annals of Mathematical Statistics}}
  \bibinfo{volume}{39}, \bibinfo{number}{4} (\bibinfo{year}{1968}),
  \bibinfo{pages}{1191--1195}.
\newblock


\bibitem[\protect\citeauthoryear{Kalofolias, Boley, and Vreeken}{Kalofolias
  et~al\mbox{.}}{2017}]%
        {kalofolias2017efficiently}
\bibfield{author}{\bibinfo{person}{Janis Kalofolias}, \bibinfo{person}{Mario
  Boley}, {and} \bibinfo{person}{Jilles Vreeken}.}
  \bibinfo{year}{2017}\natexlab{}.
\newblock \showarticletitle{Efficiently discovering locally exceptional yet
  globally representative subgroups}. In \bibinfo{booktitle}{\emph{2017 IEEE
  International Conference on Data Mining (ICDM)}}. IEEE.
\newblock


\bibitem[\protect\citeauthoryear{LeCun and Cortes}{LeCun and Cortes}{2010}]%
        {mnistdata}
\bibfield{author}{\bibinfo{person}{Yann LeCun} {and} \bibinfo{person}{Corinna
  Cortes}.} \bibinfo{year}{2010}\natexlab{}.
\newblock \showarticletitle{MNIST handwritten digit database}.
\newblock
\urldef\tempurl%
\url{http://yann.lecun.com/exdb/mnist/}
\showURL{%
\tempurl}


\bibitem[\protect\citeauthoryear{Lemmerich and Becker}{Lemmerich and
  Becker}{2019}]%
        {lemmerich2019pysubgroup}
\bibfield{author}{\bibinfo{person}{Florian Lemmerich} {and}
  \bibinfo{person}{Martin Becker}.} \bibinfo{year}{2019}\natexlab{}.
\newblock \showarticletitle{pysubgroup: Easy-to-use subgroup discovery in
  python}. In \bibinfo{booktitle}{\emph{ECML PKDD 2018}}. Springer,
  \bibinfo{pages}{658--662}.
\newblock


\bibitem[\protect\citeauthoryear{Li, Long, and Srinivasan}{Li
  et~al\mbox{.}}{2001}]%
        {li2001improved}
\bibfield{author}{\bibinfo{person}{Yi Li}, \bibinfo{person}{Philip~M Long},
  {and} \bibinfo{person}{Aravind Srinivasan}.} \bibinfo{year}{2001}\natexlab{}.
\newblock \showarticletitle{Improved bounds on the sample complexity of
  learning}.
\newblock \bibinfo{journal}{\emph{J. Comput. System Sci.}}
  \bibinfo{volume}{62}, \bibinfo{number}{3} (\bibinfo{year}{2001}),
  \bibinfo{pages}{516--527}.
\newblock


\bibitem[\protect\citeauthoryear{Llinares-L{\'o}pez, Sugiyama, Papaxanthos, and
  Borgwardt}{Llinares-L{\'o}pez et~al\mbox{.}}{2015}]%
        {llinares2015fast}
\bibfield{author}{\bibinfo{person}{Felipe Llinares-L{\'o}pez},
  \bibinfo{person}{Mahito Sugiyama}, \bibinfo{person}{Laetitia Papaxanthos},
  {and} \bibinfo{person}{Karsten Borgwardt}.} \bibinfo{year}{2015}\natexlab{}.
\newblock \showarticletitle{Fast and memory-efficient significant pattern
  mining via permutation testing}. In \bibinfo{booktitle}{\emph{Proceedings of
  the 21th ACM SIGKDD international conference on knowledge discovery and data
  mining}}. \bibinfo{pages}{725--734}.
\newblock


\bibitem[\protect\citeauthoryear{L{\"o}ffler and Phillips}{L{\"o}ffler and
  Phillips}{2009}]%
        {loffler2009shape}
\bibfield{author}{\bibinfo{person}{Maarten L{\"o}ffler} {and}
  \bibinfo{person}{Jeff~M Phillips}.} \bibinfo{year}{2009}\natexlab{}.
\newblock \showarticletitle{Shape fitting on point sets with probability
  distributions}. In \bibinfo{booktitle}{\emph{Algorithms-ESA 2009: 17th Annual
  European Symposium, Copenhagen, Denmark, September 7-9, 2009. Proceedings
  17}}. Springer, \bibinfo{pages}{313--324}.
\newblock


\bibitem[\protect\citeauthoryear{McDiarmid}{McDiarmid}{1989}]%
        {mcdiarmid1989method}
\bibfield{author}{\bibinfo{person}{Colin McDiarmid}.}
  \bibinfo{year}{1989}\natexlab{}.
\newblock \showarticletitle{On the method of bounded differences}.
\newblock \bibinfo{journal}{\emph{Surveys in combinatorics}}
  \bibinfo{volume}{141}, \bibinfo{number}{1} (\bibinfo{year}{1989}),
  \bibinfo{pages}{148--188}.
\newblock


\bibitem[\protect\citeauthoryear{Meinshausen, Maathuis, and
  B{\"u}hlmann}{Meinshausen et~al\mbox{.}}{2011}]%
        {meinshausen2011asymptotic}
\bibfield{author}{\bibinfo{person}{Nicolai Meinshausen},
  \bibinfo{person}{Marloes~H Maathuis}, {and} \bibinfo{person}{Peter
  B{\"u}hlmann}.} \bibinfo{year}{2011}\natexlab{}.
\newblock \showarticletitle{Asymptotic optimality of the Westfall-Young
  permutation procedure for multiple testing under dependence}.
\newblock \bibinfo{journal}{\emph{The Annals of Statistics}}
  (\bibinfo{year}{2011}), \bibinfo{pages}{3369--3391}.
\newblock


\bibitem[\protect\citeauthoryear{Minato, Uno, Tsuda, Terada, and Sese}{Minato
  et~al\mbox{.}}{2014}]%
        {minato2014fast}
\bibfield{author}{\bibinfo{person}{Shin-ichi Minato}, \bibinfo{person}{Takeaki
  Uno}, \bibinfo{person}{Koji Tsuda}, \bibinfo{person}{Aika Terada}, {and}
  \bibinfo{person}{Jun Sese}.} \bibinfo{year}{2014}\natexlab{}.
\newblock \showarticletitle{A fast method of statistical assessment for
  combinatorial hypotheses based on frequent itemset enumeration}. In
  \bibinfo{booktitle}{\emph{ECML PKDD 2014}}. Springer.
\newblock


\bibitem[\protect\citeauthoryear{Pellegrina, Cousins, Vandin, and
  Riondato}{Pellegrina et~al\mbox{.}}{2022}]%
        {pellegrina2022mcrapper}
\bibfield{author}{\bibinfo{person}{Leonardo Pellegrina}, \bibinfo{person}{Cyrus
  Cousins}, \bibinfo{person}{Fabio Vandin}, {and} \bibinfo{person}{Matteo
  Riondato}.} \bibinfo{year}{2022}\natexlab{}.
\newblock \showarticletitle{MCRapper: Monte-Carlo Rademacher averages for poset
  families and approximate pattern mining}.
\newblock \bibinfo{journal}{\emph{ACM Transactions on Knowledge Discovery from
  Data (TKDD)}} \bibinfo{volume}{16}, \bibinfo{number}{6}
  (\bibinfo{year}{2022}), \bibinfo{pages}{1--29}.
\newblock


\bibitem[\protect\citeauthoryear{Pellegrina, Riondato, and Vandin}{Pellegrina
  et~al\mbox{.}}{2019a}]%
        {PellegrinaRV19b}
\bibfield{author}{\bibinfo{person}{Leonardo Pellegrina},
  \bibinfo{person}{Matteo Riondato}, {and} \bibinfo{person}{Fabio Vandin}.}
  \bibinfo{year}{2019}\natexlab{a}.
\newblock \showarticletitle{Hypothesis Testing and Statistically-sound Pattern
  Mining}. In \bibinfo{booktitle}{\emph{Proceedings of the 25th ACM SIGKDD
  International Conference on Knowledge Discovery \& Data Mining}}
  \emph{(\bibinfo{series}{KDD '19})}. \bibinfo{publisher}{ACM},
  \bibinfo{address}{New York, NY, USA}, \bibinfo{pages}{3215--3216}.
\newblock


\bibitem[\protect\citeauthoryear{Pellegrina, Riondato, and Vandin}{Pellegrina
  et~al\mbox{.}}{2019b}]%
        {PellegrinaRV19a}
\bibfield{author}{\bibinfo{person}{Leonardo Pellegrina},
  \bibinfo{person}{Matteo Riondato}, {and} \bibinfo{person}{Fabio Vandin}.}
  \bibinfo{year}{2019}\natexlab{b}.
\newblock \showarticletitle{{SPuManTE}: Significant Pattern Mining with
  Unconditional Testing}. In \bibinfo{booktitle}{\emph{Proceedings of the 25th
  ACM SIGKDD International Conference on Knowledge Discovery \& Data Mining}}
  \emph{(\bibinfo{series}{KDD '19})}. \bibinfo{publisher}{ACM},
  \bibinfo{address}{New York, NY, USA}, \bibinfo{pages}{1528--1538}.
\newblock


\bibitem[\protect\citeauthoryear{Pellegrina and Vandin}{Pellegrina and
  Vandin}{2020}]%
        {pellegrina2020efficient}
\bibfield{author}{\bibinfo{person}{Leonardo Pellegrina} {and}
  \bibinfo{person}{Fabio Vandin}.} \bibinfo{year}{2020}\natexlab{}.
\newblock \showarticletitle{Efficient mining of the most significant patterns
  with permutation testing}.
\newblock \bibinfo{journal}{\emph{Data Mining and Knowledge Discovery}}
  \bibinfo{volume}{34} (\bibinfo{year}{2020}), \bibinfo{pages}{1201--1234}.
\newblock


\bibitem[\protect\citeauthoryear{Pietracaprina and Vandin}{Pietracaprina and
  Vandin}{2007}]%
        {pietracaprina2007efficient}
\bibfield{author}{\bibinfo{person}{Andrea Pietracaprina} {and}
  \bibinfo{person}{Fabio Vandin}.} \bibinfo{year}{2007}\natexlab{}.
\newblock \showarticletitle{Efficient incremental mining of top-K frequent
  closed itemsets}. In \bibinfo{booktitle}{\emph{International Conference on
  Discovery Science}}. Springer, \bibinfo{pages}{275--280}.
\newblock


\bibitem[\protect\citeauthoryear{Pollard}{Pollard}{2012}]%
        {pollard2012convergence}
\bibfield{author}{\bibinfo{person}{David Pollard}.}
  \bibinfo{year}{2012}\natexlab{}.
\newblock \bibinfo{booktitle}{\emph{Convergence of stochastic processes}}.
\newblock \bibinfo{publisher}{Springer Science \& Business Media}.
\newblock


\bibitem[\protect\citeauthoryear{Riondato and Vandin}{Riondato and
  Vandin}{2020}]%
        {riondato2020misosoup}
\bibfield{author}{\bibinfo{person}{Matteo Riondato} {and}
  \bibinfo{person}{Fabio Vandin}.} \bibinfo{year}{2020}\natexlab{}.
\newblock \showarticletitle{MiSoSouP: Mining interesting subgroups with
  sampling and pseudodimension}.
\newblock \bibinfo{journal}{\emph{ACM Transactions on Knowledge Discovery from
  Data (TKDD)}} \bibinfo{volume}{14}, \bibinfo{number}{5}
  (\bibinfo{year}{2020}), \bibinfo{pages}{1--31}.
\newblock


\bibitem[\protect\citeauthoryear{Santoro, Tonon, and Vandin}{Santoro
  et~al\mbox{.}}{2020}]%
        {santoro2020mining}
\bibfield{author}{\bibinfo{person}{Diego Santoro}, \bibinfo{person}{Andrea
  Tonon}, {and} \bibinfo{person}{Fabio Vandin}.}
  \bibinfo{year}{2020}\natexlab{}.
\newblock \showarticletitle{Mining Sequential Patterns with VC-Dimension and
  Rademacher Complexity}.
\newblock \bibinfo{journal}{\emph{Algorithms}} \bibinfo{volume}{13},
  \bibinfo{number}{5} (\bibinfo{year}{2020}), \bibinfo{pages}{123}.
\newblock


\bibitem[\protect\citeauthoryear{Shalev-Shwartz and Ben-David}{Shalev-Shwartz
  and Ben-David}{2014}]%
        {ShalevSBD14}
\bibfield{author}{\bibinfo{person}{Shai Shalev-Shwartz} {and}
  \bibinfo{person}{Shai Ben-David}.} \bibinfo{year}{2014}\natexlab{}.
\newblock \bibinfo{booktitle}{\emph{Understanding Machine Learning: From Theory
  to Algorithms}}.
\newblock \bibinfo{publisher}{Cambridge University Press}.
\newblock


\bibitem[\protect\citeauthoryear{Terada, Kim, and Sese}{Terada
  et~al\mbox{.}}{2015}]%
        {terada2015high}
\bibfield{author}{\bibinfo{person}{Aika Terada}, \bibinfo{person}{Hanyoung
  Kim}, {and} \bibinfo{person}{Jun Sese}.} \bibinfo{year}{2015}\natexlab{}.
\newblock \showarticletitle{High-speed Westfall-Young permutation procedure for
  genome-wide association studies}. In \bibinfo{booktitle}{\emph{Proceedings of
  the 6th ACM Conference on Bioinformatics, Computational Biology and Health
  Informatics}}. \bibinfo{pages}{17--26}.
\newblock


\bibitem[\protect\citeauthoryear{Terada, Okada-Hatakeyama, Tsuda, and
  Sese}{Terada et~al\mbox{.}}{2013}]%
        {TeradaOHTS13}
\bibfield{author}{\bibinfo{person}{Aika Terada}, \bibinfo{person}{Mariko
  Okada-Hatakeyama}, \bibinfo{person}{Koji Tsuda}, {and} \bibinfo{person}{Jun
  Sese}.} \bibinfo{year}{2013}\natexlab{}.
\newblock \showarticletitle{Statistical significance of combinatorial
  regulations}.
\newblock \bibinfo{journal}{\emph{Proceedings of the National Academy of
  Sciences}} \bibinfo{volume}{110}, \bibinfo{number}{32}
  (\bibinfo{year}{2013}), \bibinfo{pages}{12996--13001}.
\newblock


\bibitem[\protect\citeauthoryear{Van~Leeuwen and Knobbe}{Van~Leeuwen and
  Knobbe}{2012}]%
        {van2012diverse}
\bibfield{author}{\bibinfo{person}{Matthijs Van~Leeuwen} {and}
  \bibinfo{person}{Arno Knobbe}.} \bibinfo{year}{2012}\natexlab{}.
\newblock \showarticletitle{Diverse subgroup set discovery}.
\newblock \bibinfo{journal}{\emph{Data Mining and Knowledge Discovery}}
  \bibinfo{volume}{25} (\bibinfo{year}{2012}), \bibinfo{pages}{208--242}.
\newblock


\bibitem[\protect\citeauthoryear{Wallenius}{Wallenius}{1963}]%
        {wallenius1963biased}
\bibfield{author}{\bibinfo{person}{Kenneth~Ted Wallenius}.}
  \bibinfo{year}{1963}\natexlab{}.
\newblock \bibinfo{booktitle}{\emph{Biased sampling: the noncentral
  hypergeometric probability distribution}}.
\newblock \bibinfo{type}{{T}echnical {R}eport}.
\newblock
\urldef\tempurl%
\url{https://purl.stanford.edu/wh056vj9347}
\showURL{%
\tempurl}


\bibitem[\protect\citeauthoryear{Webb}{Webb}{2006}]%
        {webb2006discovering}
\bibfield{author}{\bibinfo{person}{Geoffrey~I Webb}.}
  \bibinfo{year}{2006}\natexlab{}.
\newblock \showarticletitle{Discovering significant rules}. In
  \bibinfo{booktitle}{\emph{Proceedings of the 12th ACM SIGKDD international
  conference on Knowledge discovery and data mining}}.
  \bibinfo{pages}{434--443}.
\newblock


\bibitem[\protect\citeauthoryear{Webb}{Webb}{2007}]%
        {webb2007discovering}
\bibfield{author}{\bibinfo{person}{Geoffrey~I Webb}.}
  \bibinfo{year}{2007}\natexlab{}.
\newblock \showarticletitle{Discovering significant patterns}.
\newblock \bibinfo{journal}{\emph{Machine learning}}  \bibinfo{volume}{68}
  (\bibinfo{year}{2007}), \bibinfo{pages}{1--33}.
\newblock


\bibitem[\protect\citeauthoryear{Webb}{Webb}{2008}]%
        {webb2008layered}
\bibfield{author}{\bibinfo{person}{Geoffrey~I Webb}.}
  \bibinfo{year}{2008}\natexlab{}.
\newblock \showarticletitle{Layered critical values: a powerful
  direct-adjustment approach to discovering significant patterns}.
\newblock \bibinfo{journal}{\emph{Machine Learning}}  \bibinfo{volume}{71}
  (\bibinfo{year}{2008}), \bibinfo{pages}{307--323}.
\newblock


\bibitem[\protect\citeauthoryear{Westfall and Young}{Westfall and
  Young}{1993}]%
        {westfall1993resampling}
\bibfield{author}{\bibinfo{person}{Peter~H Westfall} {and}
  \bibinfo{person}{S~Stanley Young}.} \bibinfo{year}{1993}\natexlab{}.
\newblock \bibinfo{booktitle}{\emph{Resampling-based multiple testing: Examples
  and methods for p-value adjustment}}.
\newblock \bibinfo{publisher}{John Wiley \& Sons}.
\newblock


\end{thebibliography}

\ifextversion
\clearpage
\newpage
\appendix

\section{Appendix}

\subsection{Relation to Quality Measures}
\label{sec:qualitymeasures}
Interesting subgroups are often identified using a \emph{quality measure}, defined by combining the \emph{generality} and the \emph{unusualness} of a pattern. The \emph{generality} $\freqp$ of a pattern $\patt$ in the dataset $\dataset$ is the fraction of samples of $\dataset$ that support $\patt$
\begin{align*}
\freqp = \frac{1}{m} \sum_{i=1}^m \ind{ \patt \in s_i } = \frac{|\covp|}{m} .
\end{align*}

For any bag of samples $B \subseteq \dataset$, define the average target value $\mu(B)$ of samples $t \in B$ as
\begin{align*}
\mu(B) = \frac{1}{|B|} \sum_{(s , \ell) \in B} \ell .
\end{align*}
The \emph{unusualness} $\unusp$ of the pattern $\patt$ on the dataset $\dataset$ is defined as the difference of the target variable of the samples $\in \covp$ and the average target in the entire data $\dataset$
\begin{align*}
\unusp = \mu(\covp) - \mu(\dataset) .
\end{align*}

The $\alpha$-quality $\nqualp{\patt,\alpha}{\dataset}$ of a pattern $\patt$ on a dataset $\dataset$ is defined as 
\begin{align*}
\nqualp{\patt,\alpha}{\dataset} = \freqp^{\alpha} \unusp.
\end{align*}

Commonly used quality measures are the $1$-quality, the $\frac{1}{2}$-quality, and the $2$-quality. 

Given a subgroup $\patt$, its quality $\tqual$ can be written as 
\begin{align*}
& \tqual = \Pr_{(s,\ell) \sim \probdist} \pars{ \patt \in s \wedge \ell = 1 } -  \Pr_{(s,\ell) \sim \probdist} \pars{ \patt \in s } \Pr_{(s,\ell) \sim \probdist} \pars{ \ell = 1 }  \\ 
& = \Pr_{(s,\ell) \sim \probdist}\pars{\patt \in s} \left( \Pr_{(s,\ell) \sim \probdist} \pars{ \ell = 1  | \patt \in s}  -  \Pr_{(s,\ell) \sim \probdist} \pars{ \ell = 1 } \right).
\end{align*} 

Note that the generality $\freqp$ is the estimate (on dataset $\dataset$) of $\Pr_{(s,\ell) \sim \probdist}\pars{\patt \in s} $, and that the unusualness $\unusp$ is the estimate  (on dataset $\dataset$) of $\Pr_{(s,\ell) \sim \probdist} \pars{ \ell = 1  | \patt \in s}  -  \Pr_{(s,\ell) \sim \probdist} \pars{ \ell = 1 }$, where $\Pr_{(s,\ell) \sim \probdist} \pars{ \ell = 1  | \patt \in s}$ is estimated by $\mu(\covp) $ and $ \Pr_{(s,\ell) \sim \probdist} \pars{ \ell = 1 }$ is  estimated by $\mu(\dataset)$. 
This shows that the $1$-quality $\nqualp{\patt,1}{\dataset}= \freqp \unusp$ corresponds to an estimate, obtained from $\dataset$, of the quality $\tqual$ of subgroup $\patt$. With this relation in mind, we can consider mining subgroups with high $1$-quality $\nqualp{\patt,1}{\dataset}$ as an heuristic for finding significant subgroups, which ignores the random fluctuations of the estimates $\freqp$ and $\unusp$.

\subsection{Comparison with Permutation Approaches}
\label{sec:appendix_comp_perm}

In this section we provide a more detailed comparison between our few-shot approach and commonly used permutation approaches.

Our few-shot approach leverages our analytical results (e.g., Theorem~\ref{thm:upperboundestcond}) to obtain high probability bounds on the maximum deviation of patterns' quality by estimating only the \emph{expectation} of maximum deviation of patterns' quality. Estimating such expectation requires a small number $c$ of resampled datasets, such as $c=10$ that we used in our experimental evaluation.

The same approach cannot be used by permutation approaches (e.g.,~\cite{llinares2015fast,pellegrina2020efficient,terada2015high}), since they are estimating the $\delta$-quantile of the distribution of the maximum deviation, that is, the value $q$ for which the maximum deviation is below $q$ with probability $\delta$, for a (relatively) small value of $\delta$. Accurately estimating such quantile requires many more permutations than estimating the expectation (as done by our approach). For example, if only $10$ permutations (e.g., corresponding  to the value $c$ used in our experiments)  are used to estimate the $\delta$-quantile, with $\delta = 0.05$, the WY procedure returns the maximum deviation over the $10$ permutations 
(i.e., the element in position $\lceil \delta \cdot 10 \rceil = 1$ in the list of deviations, sorted in decreasing order).
This implies that, 
with probability $> \frac{1}{2}$, the FWER will not be controlled at level $\delta$ (since the probability that the deviation of one permutation will be above the $\delta$-quantile is $0.95$, and the probability that all deviations are above the $\delta$-quantile is $0.95^{10} \approx 0.599 > \frac{1}{2}$). 
Moreover, with probability $> \frac{1}{3}$, the FWER will not be controlled even at level $2\delta = 0.1$ (since $0.9^{10} \approx 0.3487 > \frac{1}{3}$).  For such a reason, previous works suggest to use \emph{at least} $10^{3}$ permutations for permutation approaches (as we do in our experiments), while $10^{4}$ is the suggested number of permutations to have a stable FWER estimation (\cite{terada2015high,llinares2015fast,pellegrina2020efficient} all use $10^{4}$).

\subsection{Proofs of \Cref{sec:algo_conditional}}
This section presents additional proofs for the results of \Cref{sec:algo_conditional}.

First, we need the following technical result. 
\begin{theorem}[McDiarmid's inequality~\cite{mcdiarmid1989method}]
  \label{thm:mcdiarmid}
  Let $\mathcal{Y}$ be a domain, and let $g : \mathcal{Y}^m \rightarrow \R$
  be a function such that, for each $i$, $1\le i\le m$, there is a
  nonnegative constant $c_i$ such that:
  \begin{equation*}
    \sup_{\substack{\langle x_1,\dotsc,x_m\rangle \in
  \mathcal{Y}^m\\x_i'\in\mathcal{Y}}}
    \abs{g(x_1,\dotsc,x_m) -
    g(x_1,\dotsc,x_{i-1},x'_i,x_{i+1},\dotsc,x_m)} \le c_i .
  \end{equation*}
  Let $x_1,\dotsc,x_m$ be $m$ \emph{independent} random variables such that $\langle x_1,\dotsc,x_m\rangle \in
  \mathcal{Y}^m$. Then, for $C=\sum_{i=1}^m c_i^2$, it holds
  \[
    \Pr \Bigl( \E[g] > g(x_1,\dotsc,x_m) + t \Bigr) \le e^{-2t^2/C} .
  \]
\end{theorem}

\begin{proof}[Proof of \Cref{thm:upperboundestcond}]
First, we note that in the conditional setting it holds $\hat{\mu} = \check{\mu} = \bar{\mu}$. Then, from the fact $\lang^\star \subseteq \lang$, we have  
\begin{align*}
&\E_{\mathbf{v} \sim I(\bar{\mu})} \sqpars{ \sup_{\patt \in \lang^\star} \frac{1}{m} \sum_{i=1}^m f_{\patt}(s_i)( \mathbf{v}_i - \bar{\mu} ) } \\
&\leq \E_{\mathbf{v} \sim I(\bar{\mu})} \sqpars{ \sup_{\patt \in \lang} \frac{1}{m} \sum_{i=1}^m f_{\patt}(s_i)( \mathbf{v}_i - \bar{\mu} ) } 
= \E_{\resset} \bigr[ \tilde{d}(\resset, \check{\mu}) \bigl] ,
\end{align*}
therefore we focus on the concentration of $\tilde{d}(\resset, \check{\mu})$ around its expectation taken w.r.t. $\resset$. 
Our proof is based on McDiarmid's inequality (\Cref{thm:mcdiarmid}).
Define the function $g(\resset) = \tilde{d}(\resset, \check{\mu})$, and note that modifying any $\xi_{i,j}$, for any pair $i,j$, changes $g(\resset)$ by at most $1/(cm)$.
Therefore, defining $C=\sum_i \sum_j (1/(cm))^2 = 1/(cm)$, from \Cref{thm:mcdiarmid} it holds 
\begin{align*}
\Pr \Biggl( \E_{\resset} \bigr[ \tilde{d}(\resset, \check{\mu}) \bigl] > \tilde{d}(\resset, \check{\mu}) + \sqrt{ \frac{ \log \bigl( \frac{4}{\delta} \bigr) }{2cm} }  \Biggr) \leq \delta/4 .
\end{align*}
\end{proof}

\subsection{Proofs of \Cref{sec:algo_unconditional}}
\label{sec:appxproofsuncond}
This Section presents the proofs for the results of \Cref{sec:algo_unconditional}. 

\begin{proof}[Proof of Lemma \ref{thm:qualestimatemu}]
Since $\mu(\dataset)$ is the average of $m$ independent and bounded random variables, Hoeffding's and Bernstein's inequalities \cite{boucheron2013concentration} yield, respectively, that
\begin{align*}
| \mu - \mu(\dataset) | &\leq \sqrt{\frac{\ln \pars{ \frac{8}{\delta} } }{2m}} , \\
| \mu - \mu(\dataset) | &\leq \sqrt{\frac{2  \mu(\dataset) \ln \pars{ \frac{8}{\delta} } }{m}} + \frac{2 \ln \pars{ \frac{8}{\delta} } }{m} , 
\end{align*}
hold simultaneously with probability $\geq 1 - \delta/4$. 
The statement follows from the observation that their minimum is $\leq \varepsilon_T$. 
\end{proof}

\begin{proof}[Proof of \Cref{thm:bound_dev}]
Recall that   
$\resset = \brpars{\dataset^\star_1 , \dots, \dataset^\star_c}$ is a collection of $c \geq 1$ i.i.d. resampled datasets, each obtained by  resampling the target labels of $\dataset$ while maintaining the same features of $\dataset$. That is, each resampled dataset is $\dataset^\star_j = \brpars{ (s_1 , \xi_{1,j}) , \dots (s_m , \xi_{m,j}) }$, where
\begin{align*}
\xi_{i,j} \sim Bern(p) 
, \forall i \in [1,m] , \forall j \in [1,c] ,
\end{align*}
and $Bern(p)$ is the Bernoulli distribution with parameter $p$. In the proof we set $p = \mu = \E_\dataset[\mu(\dataset)]$. 
Then, we define the collection of resampled datasets 
$\resseth = \{\hat{\dataset}^\star_1 , \dots, \hat{\dataset}^\star_c\}$ similarly to $\resset$, where the parameter $\mu$ is replaced by its upper bound $\hat{\mu}$ (note that the roles of $\resset$ and $\resseth$ are swapped in the statement and in Alg.~1).

Let $Y = \E_{\dataset} \bigl[ \sup_{\patt \in \lang^\star}\{ \squalp \} \bigr]$,
and observe that it holds $\sup_{\patt \in \lang^\star} Var_\dataset(\squalp) \leq \nu$. 
We apply Bousquet's inequality (Theorem 12.5 of \cite{boucheron2013concentration}) to prove that 
\begin{align*}
\sup_{\patt \in \lang^\star}\{ \squalp \} \leq Y + \sqrt{\frac{2 \unionbdelta
    \left( \nu + 2Y \right)}{m}}
        + \frac{ \unionbdelta}{3m} ,
\end{align*} 
with probability $\geq 1 - \delta/4$. 
We now show how to upper bound $Y$. 
First, note that 
\begin{align*}
\E_{\dataset} \bigl[ \sup_{\patt \in \lang^\star}\{ \squalp \} \bigr] 
&=  \E_{\dataset^\star_j} \bigl[ \sup_{\patt \in \lang^\star}\{ \squal{\patt}{\dataset^\star_j} \} \bigr] \\
&= \E_{\resset} \biggl[ \frac{1}{c} \sum_{j=1}^c \sup_{\patt \in \lang^\star}\{ \squal{\patt}{\dataset^\star_j} \} \biggr]
\end{align*}
by definition of $\lang^\star$ and $\resset$. 
Then, it follows that 
\begin{align*}
\E_{\resset} \biggl[ \frac{1}{c} \sum_{j=1}^c \sup_{\patt \in \lang^\star}\{ \squal{\patt}{\dataset^\star_j} \} \biggr] 
\leq \E_{\resset} \biggl[ \frac{1}{c} \sum_{j=1}^c \sup_{\patt \in \lang}\{ \squal{\patt}{\dataset^\star_j} \} \biggr] ,
\end{align*}
since $\lang^\star \subseteq \lang$. 
We now prove bounds to the concentration of $\E_{\resset} \bigl[ \frac{1}{c} \sum_{j=1}^c \sup_{\patt \in \lang}\{ \squal{\patt}{\dataset^\star_j} \} \bigr] $ w.r.t. the set of features $\A$. 
Let $\dataset^\star = (\A , \T^\star)$ be a generic $\dataset^\star \in \resset$,
with $\T^\star = \{ \xi_1 , \dots , \xi_m  \}$. 
Define the random variable $Z$ as
\begin{align*}
Z = \E_{\T^\star} \bigl[ \sup_{\patt \in \lang}\{ \squal{\patt}{\dataset^\star} \} \bigr] ,
\end{align*}
where the expectation is w.r.t. $\T^\star$, conditioning on $\A$. 
To show the concentration of $Z$ w.r.t. its expectation $\E_\A[Z]$, 
we prove that $Z$ is a self-bounding function \cite{boucheron2013concentration}.
Define the random variable $Z_j$, for $j \in [1,m]$, as
\begin{align*}
Z_j = \E_{\T^\star} \biggl[ \sup_{\patt \in \lang} \biggl\{ \frac{1}{m} \sum_{i=1 , i \neq j}^m f_\patt(s_i) (\xi_i - \mu) \biggr\} \biggr] 
\end{align*}
First, note that $Z \geq 0$:
\begin{align*}
Z = \E_{\T^\star} \bigl[ \sup_{\patt \in \lang}\{ \squal{\patt}{\dataset^\star} \}  \bigr] 
 \geq  \sup_{\patt \in \lang} \bigl\{ \E_{\T^\star} \bigl[ \squal{\patt}{\dataset^\star}  \bigr] \bigr\} \geq 0. 
\end{align*}
Then, we prove that $Z - Z_j \geq 0$. Define $\hat{f}$ as one of the functions that attain the supremum for $Z_j$, for any choice of $\T^\star$. 
We have 
\begin{align*}
Z_j - Z 
& \leq Z_j - \E_{\T^\star} \biggl[ \frac{1}{m} \sum_{i=1}^m \hat{f}(s_i)(\xi_i - \mu) \biggr] \\
& = \E_{\xi_j} \biggl[ - \frac{1}{m} \hat{f}(s_j)(\xi_j - \mu) \biggr] = 0.
\end{align*}
We now show that $Z - Z_j \leq \mu (1 - \mu)/m$. We have
\begin{align*}
Z & \leq \E_{\T^\star} \biggl[  \sup_{\patt \in \lang} \biggl\{ \frac{1}{m} \sum_{i=1 , i \neq j} f_\patt (s_i)( \xi_i - \mu ) \biggr\} \\ &\; \; \; \; \; \; \;  + \sup_{\patt \in \lang} \brpars{ \frac{1}{m} f_\patt (s_j)( \xi_j - \mu )  } \biggr] \\
 & \leq Z_j + \frac{ \mu (1 - \mu) }{m} .
\end{align*}
We now prove that $\sum_{j=1}^m Z - Z_j \leq Z$.
It holds 
\begin{align*}
\sum_{j=1}^m Z_j
&\geq \E_{\T^\star} \biggl[ \sup_{\patt \in \lang} \biggl\{ \frac{1}{m} \sum_{j=1}^m \sum_{i=1 , i \neq j}^m f_\patt(s_i) (\xi_i - \mu) \biggr\} \biggr] \\
&= \E_{\T^\star} \biggl[ \sup_{\patt \in \lang} \biggl\{ \frac{m-1}{m} \sum_{i=1}^m f_\patt(s_i) (\xi_i - \mu) \biggr\} \biggr] \\
&= (m-1) Z .
\end{align*}
As $Z$ is a self-bounding function, we have that
\begin{align*}
\Pr\pars{ \E[Z] - Z \geq q } \leq \exp\pars{ \frac{- m q^2}{2 \mu(1-\mu) \E[Z] } } .
\end{align*}
Imposing the r.h.s. $\leq \delta/4$, solving for $q$, and finding the fixed point of the inequality we obtain $\hat{d}$. 
Let $ \devest =  \tilde{d}(\resset , \mu(\dataset))$. 
We apply McDiarmid inequality (\Cref{thm:mcdiarmid}) to show that $\E_{\resset}[ \devest \given \A ] \leq \hat{r}$ with probability $\geq 1-\delta/4$.
In fact, define the function $g(\resset) = \devest$, and note that modifying any $\xi_{i,j}$, for any pair $i,j$, changes $g(\resset)$ by at most $1/(cm)$.
Therefore, defining $C=\sum_i \sum_j (1/(cm))^2 = 1/(cm)$, the upper bound $\hat{r}$ to $\E_{\resset}[ \devest \given \A ]$ holds by \Cref{thm:mcdiarmid}. 
We now need to prove that 
the upper bound $\hat{d}$ computed using $\resseth$ is (probabilistically) not smaller than using $\resset$. 
This is equivalent to show that it holds, for all $x$, 
\begin{align*}
\Pr_{\resseth} \pars{ \devesth > x } \geq \Pr_{\resset} \pars{ \devest > x } .
\end{align*}
Equivalently, the probability of underestimating $\E_{\resset}[\devest]$ using $\devest$ does not increase when using $\devesth$. 
Since the two probabilities are taken w.r.t. to two different sample spaces, it is not possible to compare them directly. 
Therefore, we build an appropriate coupling between the two distributions. 
Define an $m \times c$ matrix $v$ of $mc$ i.i.d. Bernoulli random variables, such that $\Pr(v_{i,j} = 1) = \varepsilon_T/(1-\mu)$ for all $i,j$. 
(Note that we assume $0 < \mu < 1$ and $0\leq \check{\mu} \leq \mu \leq \hat{\mu} \leq 1$, otherwise the statement holds trivially.)  
We observe that
\begin{align*}
&\Pr_{\resseth} \pars{ \devesth > x } 
= \Pr_{\resseth} \pars{ \frac{1}{c} \sum_{j=1}^c \sup_{\patt \in \lang} \brpars{ \frac{1}{m} \sum_{i=1}^m f_\patt (s_i) \pars{ \hat{\xi}_{i,j} - \mu } } > x } ,
\end{align*}
where $\hat{\xi}_{i,j}$ are i.i.d. Bernoulli with $\Pr(\hat{\xi}_{i,j} = 1) = \hat{\mu}$ for all $i,j$. 
We build the following coupling between the distributions of $\resseth$ and $\resset$, using the fact that $\hat{\xi}_{i,j} \sim \max\{\xi_{i,j} , v_{i,j} \}$.  
This allows us to obtain the lower bound stated above:
\begin{align*}
&\Pr_{\resseth} \pars{ \frac{1}{c} \sum_{j=1}^c \sup_{\patt \in \lang} \brpars{ \frac{1}{m} \sum_{i=1}^m f_\patt (s_i) \pars{ \hat{\xi}_{i,j} - \mu } } > x } \\
&= \Pr_{\resset , v} \pars{ \frac{1}{c} \sum_{j=1}^c \sup_{\patt \in \lang} \brpars{ \frac{1}{m} \sum_{i=1}^m f_\patt (s_i) \pars{ \max\{\xi_{i,j} , v_{i,j} \} - \mu } } > x } \\
&\geq \Pr_{\resset , v} \pars{ \frac{1}{c} \sum_{j=1}^c \sup_{\patt \in \lang} \brpars{ \frac{1}{m} \sum_{i=1}^m f_\patt (s_i) \pars{ \xi_{i,j} - \mu } } > x } \\
&= \Pr_{\resset} \pars{ \devest > x } .
\end{align*}
Then, it is immediate to observe that 
\begin{align*}
\Pr_{\resseth} \pars{ \tilde{d}(\resseth , \check{\mu}) > x } \geq \Pr_{\resseth} \pars{ \devesth > x } ,
\end{align*}
since $\mu$ is replaced by its lower bound $\check{\mu}$ in the definition of $\devesth$, as $\tilde{d}(\resseth , \check{\mu}) \geq \devesth$ for all $\resseth$. 
The statement follows observing that all other quantities in the definition of $\varepsilon$ are constants independent of $\mu$, and from an union bound over the $3$ concentration bounds considered in the proof, and the event $\qtm{ |\mu - \mu(\dataset) | \leq \varepsilon_T }$, each of them true with probability $\geq 1 - \delta/4$. 
\end{proof}

\subsection{Power Analysis}
\label{sec:app_power}

In this section we prove the results on the power of \algname\ stated in Sections~\ref{sec:algo_conditional} and~\ref{sec:algo_unconditional}.

We first provide a probabilistic upper bound to $\tilde{d}(\resset , \check{\mu})$, the estimate of the supremum deviation of false discoveries computed by Algorithm~\ref{algo:main} (in line~\ref{line6}). 
This result can be applied to general languages; 
we then show how to apply it to the language of subgroups. 
We define $N_\lang(\dataset)$ as the number of \emph{distinct projections} of the language $\lang$ on the dataset $\dataset$:
\begin{align*}
N_\lang(\dataset) = \abs{ \brpars{ \{ i : \patt \in s_i \} , \patt \in \lang } } .
\end{align*}
Note that, differently from $|\lang|$, $N_\lang(\dataset)$ is always a finite value (a trivial upper bound is $N_\lang(\dataset) \leq 2^m$). 

\begin{theorem}
\label{thm:upperbounddestimate}
Let $\lambda \in (0,1)$,
and $\hat{\omega} = \hat{\mu} (1 - \check{\mu}) \sup_{\patt \in \lang} \freqp $. 
The value of $\tilde{d}(\resset , \check{\mu})$ computed by \algname\ in line~\ref{line6} of Algorithm~\ref{algo:main} is 
\begin{align*}
\tilde{d}(\resset , \check{\mu}) \leq \sqrt{ \frac{ 2 \hat{\omega} \ln(N_\lang(\dataset)) }{ m } } + \sqrt{ \frac{ \ln(\frac{1}{\lambda}) }{ 2cm } } + \frac{  \ln(N_\lang(\dataset)) }{ 3m }
\end{align*}
with probability at least $1 - \lambda$.
\end{theorem}
\begin{proof}
To obtain the statement, we first prove an upper bound to the expectation 
$\E_{\mathbf{v} \sim I(\hat{\mu})} [ \tilde{d}(\resset , \check{\mu})] $, 
that is 
\begin{align*}
\E_{\mathbf{v} \sim I(\hat{\mu})} [ \tilde{d}(\resset , \check{\mu})] 
\leq \sqrt{ \frac{ 2 \hat{\omega} \ln(N_\lang(\dataset)) }{ m } } + \frac{ \ln(N_\lang(\dataset)) }{ 3m } ,
\end{align*}
and then conclude with a concentration bound for $\tilde{d}(\resset , \check{\mu})$ w.r.t. to its expectation $\E_{\mathbf{v} \sim I(\hat{\mu})} [ \tilde{d}(\resset , \check{\mu})] $, using analogous derivations of the proof of \Cref{thm:upperboundestcond}. 
For any $\patt \in \lang$, define $X_\patt = \frac{1}{m} \sum_{i=1}^m f_{\patt}(s_i)( \mathbf{v}_i - \check{\mu} )$. 
Let two patterns $\patt_1, \patt_2$ such that the projection of $\patt_1$ on $\dataset$ is equal to the projection of $\patt_2$ on $\dataset$, i.e., it holds  
$\{ i : \patt_1 \in s_i \} = \{ i : \patt_2 \in s_i \}$.
This implies that $\sup_{i \in \{1,2\}} X_{\patt_i} = X_{\patt_1} $.
Therefore, we can rewrite the supremum within $\E_{\mathbf{v} \sim I(\hat{\mu})} [ \tilde{d}(\resset , \check{\mu})]$ over the set of patterns with distinct projections, recalling that the number of such distinct projections is $N_\lang(\dataset)$. 
Now, for any $\patt \in \lang$, Bernstein's inequality (Theorem 2.10 in \cite{boucheron2013concentration}) implies that 
$X_\patt$ is a sub-gamma random variable (Section 2.4 \cite{boucheron2013concentration}), such that $X_\patt \in \Gamma_+(u , b)$ with $u = \hat{\omega}/m$ and $b = 1/(3m)$, since it is an average of i.i.d. random variables $f_{\patt}(s_i)( \mathbf{v}_i - \check{\mu} )$ that are bounded in the interval $[- \check{\mu} , 1 - \check{\mu}]$ and have variance $\leq \hat{\omega}$. 
Consequently, we apply a maximal inequality (Corollary 2.6 in \cite{boucheron2013concentration}) to upper bound the expected maximum of sub-gamma random variables, obtaining that 
$\E[ \max_{\patt} X_\patt ] \leq \sqrt{ 2u\ln(N_\lang(\dataset)) } + b \ln(N_\lang(\dataset))$. 
Note that the upper bound to the expectation given above holds.
The statement follows from the application of \Cref{thm:mcdiarmid} to $\tilde{d}(\resset , \check{\mu})$,
following the same steps of the proof of \Cref{thm:upperboundestcond}. 
\end{proof}

To upper bound $N_\lang(\dataset)$ for the language of subgroups, we prove the following.
Note that we focus on the case of subgroups with continuous features, since every categorical feature $f_{c}$ can be converted to a discrete one $f_{d}$ (assigning a random order to the distinct elements), and observing that each equality condition $\qtm{f_{c} = a}$ is equivalent to the interval $\qtm{f_{d} \in [a-x , a+x]}$ for some $x > 0$. 
Therefore, the language of subgroups over continuous features contains the language over datasets with both continuous and categorical features, thus all results for the former apply to the latter. 
Note the reverse direction is not true in general. 
\begin{lemma}
\label{thm:distinctsubgroups}
Let $\lang$ be the language of subgroups composed by conjunctions with at most $z$ conditions over $d$ continuous features. Then, it holds $N_\lang(\dataset) \leq \pars{ \frac{e^3 d m^2 }{4z^3} }^z$.
\end{lemma}
\begin{proof}
Consider a dataset $\dataset$ with $d$ continuous features, and let $[1,d]$ be the indices of these features.
Let any $A \subseteq [1,d]$ with $|A| = v \leq z$, and define the language $\lang_A \subseteq \lang$ as all subgroups with $v$ conjunction terms involving conditions on the features of $A$, such as inequalities or intervals. 
Equivalently, we can see the projection of $\lang_A$ over the transactions of $\dataset$ as the class of axis-aligned rectangles in $\R^v$. 
It is known that the VC-dimension of this class is $2v$ (Problem 6.5 in \cite{ShalevSBD14}).
Therefore, from Sauer-Shelah-Perles' Lemma (Lemma 6.10 in \cite{ShalevSBD14}),
the number $N_{\lang_A}(\dataset)$ of distinct projections of $\lang_A$ on $\dataset$ is $N_{\lang_A}(\dataset) \leq \sum_{i=1}^{2v} \binom{m}{i} \leq  \pars{ \frac{e m }{ 2v } }^{2v}$. 
From an union bound, 
\begin{align*}
N_{\lang}(\dataset) 
\leq \sum_A N_{\lang_A}(\dataset) 
\leq \sum_{v=1}^z \binom{d}{v} \pars{ \frac{e m }{ 2v } }^{2v}
\leq \pars{ \frac{e m }{ 2z } }^{2z} \pars{ \frac{e d }{ z } }^{z} ,
\end{align*}
obtaining the statement. 
\end{proof}
Combining \Cref{thm:distinctsubgroups} and \Cref{thm:upperbounddestimate}, we obtain the following Corollary.
\begin{corollary}
\label{thm:corupperbounddevsub}
Let $\lang$ be the language of subgroups composed by conjunctions with at most $z$ conditions over $d$ continuous features.
Then the value $\tilde{d}(\resset , \check{\mu})$ computed by \algname\ in line~\ref{line6} of Algorithm~\ref{algo:main} is 
\begin{align*}
\tilde{d}(\resset , \check{\mu}) \leq \sqrt{ \frac{ 2 \hat{\omega} z \ln(\frac{e^3 d m^2 }{4z^3}) }{ m } } + \sqrt{ \frac{ \ln(\frac{1}{\lambda}) }{ 2cm } } + \frac{ z \ln(\frac{e^3 d m^2 }{4z^3}) }{ 3m }
\end{align*}
with probability at least $1 - \lambda$.
\end{corollary}

\subsubsection{Power analysis of \algnamec}
To prove \Cref{thm:power_conditional}, regarding the power of \algnamec, we first describe the model we assume for the distribution of the alternative hypotheses, i.e., the set of patterns with $\tqual > 0$.
We assume that the quality of patterns correlated with the target follows the Wallenius' noncentral hypergeometric distribution \cite{wallenius1963biased}, a generalization of the hypergeometric distribution that allows to model a biased random sampling of a contingency table with fixed marginals. 
We define the model $W_n(\{m_i , w_i\})$ that describes the distribution of a sequence of binary random variables $\ell_1 , \dots , \ell_n$ for the weighted sampling of $n$ elements from a set of $m_1$ items with label $1$ and $m_0$ items with label $0$; 
the parameters $w_0$ and $w_1$, with $1 \leq w_0 < w_1 $, are respectively the weights of items with label $0$ and $1$. 
The fact $w_1 > w_0$ expresses the bias toward sampling items with label $1$. 
The first element $\ell_1$ is sampled according to the weighted proportion of the items, such that
\begin{align*}
\Pr \pars{ \ell_1 = 1 } = \frac{ m_1 w_1 }{m_1 w_1 + m_0 w_0} .
\end{align*}
The second element $\ell_2$ is taken according to the weighted proportion of the remaining items, therefore dependending on the outcome of the first choice. 
In general, we obtain that 
\begin{align*}
\Pr \pars{ \ell_i = 1 } = \frac{ m_1^i w_1 }{m_1^i w_1 + m_0^i w_0} ,
\end{align*}
where $m_1^i = m_1 - \sum_{j=1}^{i-1} \ell_j$ and $m_0^i = m_0 - \sum_{j=1}^{i-1} (1-\ell_j)$ are, respectively, the number of remaining items with label $1$ and label $0$, i.e., the ones that are not sampled in previous steps. 
We may observe that, as the bias goes to $0$ (i.e., $w_1 \rightarrow w_0$), this distribution converges to the (standard) hypergeometric distribution. 
We remark that a direct calculation of mean and the probability mass function of the noncentral hypergeometric distribution above, therefore the computation of exact tail bounds, is extremely unwieldy. 
Moreover, the random variables $\{ \ell_i , i \in [1,n] \}$ are clearly not independent, therefore standard concentration results (e.g., Chernoff-Hoeffding bounds) do not apply directly.

To overcome these issues, we leverage an advanced concentration bound for martingales, which also applies to random variables that are not necessarely independent \cite{dubhashi2009concentration}. 
We use the following version of the method of bounded differences \cite{mcdiarmid1989method}. 
For a given set of random variables $X_1 , \dots , X_n$, we use $\mathbf{X}_{i}$ to denote the set $\{ X_j , j \in [1 , i] \}$. 
\begin{definition}
A function $f$ satisfies the \emph{Averaged Lipschitz Condition} (ALC) with parameters $c_i, i \in [n]$, with respect to the random variables $X_1 , \dots , X_n$ if for any $a_i , a^\prime_i$,
\begin{align*}
\abs{ \E\sqpars{ f \mid \mathbf{X}_{i-1} , X_i = a_i } - \E\sqpars{ f \mid \mathbf{X}_{i-1} , X_i = a^\prime_i } } \leq c_i,
\end{align*}
for $1 \leq i \leq n$.
\end{definition}
The following result establishes concentration bounds for functions that satisfy the ALC.
\begin{theorem}[Corollary 5.1 \cite{dubhashi2009concentration}]
\label{thm:alcconcentration}
Let $f$ satisfy the ALC with parameters $c_i, i \in [n]$, with respect to the random variables $X_1 , \dots , X_n$, and let $C = \sum_{i=1}^n c_i^2$.
Then it holds
\begin{align*}
\Pr\pars{ f > \E[f] + t } , \Pr\pars{ f < \E[f] - t } \leq \exp ( -2t^2/C ) .
\end{align*}
\end{theorem}

We now prove that the sum $\sum_{i=1}^n \ell_i$ is a function that satisfy the ALC, thus is sharply concentrated toward its expectation.
\begin{theorem}
\label{thm:concentrationboundnoncentral}
Let $\{ \ell_i , i \in [1,n] \}$ be a set of random variables distributed according to $W_n(\{m_i , w_i\})$.
Denote $X_i = \ell_i$ and the function $f = \sum_{i=1}^n X_i$.
Then, it holds
\begin{align*}
\Pr\pars{ f > \E[f] + t } , \Pr\pars{ f < \E[f] - t } \leq \exp ( -2t^2/n ) .
\end{align*}
\end{theorem}
\begin{proof}
We prove that $f$ satisfy the ALC.
For any $i \in [1,n]$, 
\begin{align*}
&\E \sqparsmid{ \sum_{j=1}^n X_j \mid \mathbf{X}_{i-1} , X_i = 1 }
- 
\E \sqparsmid{ \sum_{j=1}^n X_j \mid \mathbf{X}_{i-1} , X_i = 0 } \\
&= 1 + \E \sqparsmid{ \sum_{j=i+1}^n X_j \mid \mathbf{X}_{i-1} , X_i = 1 }
- 
\E \sqparsmid{ \sum_{j=i+1}^n X_j \mid \mathbf{X}_{i-1} , X_i = 0 } \\
&\leq 1 + \E \sqparsmid{ \sum_{j=i+1}^n X_j \mid \mathbf{X}_{i-1} , X_i = 0 }
- 
\E \sqparsmid{ \sum_{j=i+1}^n X_j \mid \mathbf{X}_{i-1} , X_i = 0 } = 1.
\end{align*}
We then prove the other direction:
\begin{align*}
&\E \sqparsmid{ \sum_{j=1}^n X_j \mid \mathbf{X}_{i-1} , X_i = 0 }
- 
\E \sqparsmid{ \sum_{j=1}^n X_j \mid \mathbf{X}_{i-1} , X_i = 1 } \\
&= \E \sqparsmid{ \sum_{j=i+1}^n X_j \mid \mathbf{X}_{i-1} , X_i = 0 }
- 1 - 
\E \sqparsmid{ \sum_{j=i+1}^n X_j \mid \mathbf{X}_{i-1} , X_i = 1 } \\
&\leq 1 + \E \sqparsmid{ \sum_{j=i+1}^n X_j \mid \mathbf{X}_{i-1} , X_i = 1 }
- 1 - 
\E \sqparsmid{ \sum_{j=i+1}^n X_j \mid \mathbf{X}_{i-1} , X_i = 1 } \\
&= 0. 
\end{align*}
We conclude that $f$ satisfy the ALC with $c_i = 1$, for all $1 \leq i \leq n$,
and the concentration bounds follow from \Cref{thm:alcconcentration}.
\end{proof}
Using the result above, we prove a probabilitic lower bound to the observed quality $\qualp$ of a pattern $\patt$.
For any $\patt \in \lang$, let $n = m \freqp$ be the number of transactions of $\dataset$ where $\patt$ is supported, and assume w.l.o.g. that $\{ i : \patt \in s_i , i \in [1,m] \} = [1,n]$. 
We define $X_i = \ell_i$ for all $i \in [1,n]$, where $\ell_i$ follows the biased sampling distribution described above. 
\begin{proposition}
\label{thm:lowerboundqualp}
It holds
\begin{align*}
\Pr \pars{ \qualp \leq \tqual - t } \leq \exp \pars{ \frac{- 2 m t^2 }{ \freqp } } .
\end{align*}
\end{proposition}
\begin{proof}
Define the function $g = \frac{1}{m}\sum_{i=1}^n (\ell_i - \mu(\dataset))$.
Considering the function $f$ in the statement of \Cref{thm:concentrationboundnoncentral},
we observe that it holds $g = f/m-\mu(\dataset)$,
and that $n = m \freqp$. 
From these observations, we apply \Cref{thm:concentrationboundnoncentral} to the function $f = m (g + \mu(\dataset))$, obtaining the statement after simple manipulations of the r.h.s..
\end{proof}

We now extend \Cref{thm:lowerboundqualp} to obtain a bound valid simoultaneously for all patterns of the language $\lang$.

\begin{proposition}
\label{thm:lowerboundqualpunif}
Define $\hat{\mathsf{f}}(\dataset) = \sup_\patt \freqp$. 
With probability $\geq 1 - \delta$, it holds
\begin{align*}
\qualp \geq \tqual - \sqrt{ \frac{ 2 \hat{\mathsf{f}}(\dataset) \ln(N_\lang(\dataset) / \delta ) }{m} } , \forall \patt \in \lang .
\end{align*}
\end{proposition}
\begin{proof}
Define the event 
\begin{align*}
E_\patt = \qtm{ \qualp \leq \tqual - \sqrt{ \frac{ 2 \hat{\mathsf{f}}(\dataset) \ln(N_\lang(\dataset) / \delta) }{m} } } .
\end{align*} 
The statement holds if $\Pr ( \cup_\patt E_\patt ) \leq \delta$.
Now, denote two patterns $\patt_1, \patt_2$, such that they are supported by the same set of transactions: 
$\{ i : \patt_1 \in s_i \} = \{ i : \patt_2 \in s_i \}$.
This implies that ${\qual{\patt_1}{\dataset}} = {\qual{\patt_2}{\dataset}}$ for all possible labels, but also that $\mathsf{q}_{\patt_1} = \mathsf{q}_{\patt_2}$, since the two qualities are functions of the same set of labels.
From these observations, it holds $E_{\patt_1} \cup E_{\patt_2} = E_{\patt_1}$.
Therefore, the union over all patterns can be replaced by the union over the ones with distinct projections, whose number is at most $N_\lang(\dataset)$. 
From an union bound over these events, 
and using the tail bound of \Cref{thm:lowerboundqualp} with the fact $\freqp \leq \hat{\mathsf{f}}(\dataset)$, we have
\begin{align*}
\Pr \pars{ \cup_\patt E_\patt } \leq N_\lang(\dataset) \exp \bigl( - 2 m t^2 / \hat{\mathsf{f}}(\dataset)  \bigr) .
\end{align*}
Imposing the r.h.s. of the inequality to be $\leq \delta$, and solving for $t$, we obtain the statement.
\end{proof}

Finally, we prove the power guarantees for \algnamec.

\begin{proof}[Proof of \Cref{thm:power_conditional}]
To obtain the statement, we show that the condition $\tqual \geq \xi$, for some $\xi > 0$ sufficiently large, implies that $\patt$ is reported in output with probability $\geq 1-\delta$.
In fact, we have that the implication 
$\tqual \geq \xi 
\implies \qualp \geq \xi - \varepsilon_1 $ 
holds with probability $\geq 1 - \delta/2$ for all $\patt \in \lang$, 
where $\varepsilon_1 = \sqrt{ \frac{ 2 \hat{\mathsf{f}}(\dataset) z \ln(\frac{e^3 m^2 d}{2z^3 \delta}) }{m} }$ is obtained from \Cref{thm:lowerboundqualpunif} (replacing $\delta$ by $\delta/2$, and using the upper bound to $N_\lang(\dataset)$ from \Cref{thm:distinctsubgroups}). 
Consequently, we seek a sufficient condition to ensure that $\patt$ is reported in output;
this condition is $\qualp \geq \xi - \varepsilon_1 \geq \varepsilon$, and also $\xi \geq \varepsilon_1 + \varepsilon$, where $\varepsilon$ is as defined in \Cref{thm:fwerguaranteesfsrc}.
Combining this inequality with the upper bound to $\tilde{d}(\resset, \check{\mu})$ of \Cref{thm:upperbounddestimate} (setting $\lambda = \delta/2$), we obtain the lower bound to $\tqual$ in the statement, which holds with probability $\geq 1 - \delta$ from a union bound. 
\end{proof}

\subsubsection{Power analysis of \algnameu}

The following is a restatement of Theorem~5 of~\cite{li2001improved},
that provides uniform convergence bounds for families of functions with bounded pseudodimension \cite{pollard2012convergence,ShalevSBD14}. 
\begin{theorem}
\label{thm:pseudoappxthm}
Let $\F$ be a family of functions from a domain $\X$ to $[a,b] \subset \R$ with pseudodimention $P(\F) \leq r$. 
Let $\sample = \{ s_{1} , \dots , s_{m}\}$ be a random sample of size $m$ taken i.i.d. from a distribution $\probdist$. It holds  
\begin{align*}
\Pr\pars{ \sup_{f \in \F} \left\lvert \E_{\sample} \sqpars{ \frac{1}{m} \sum_{i=1}^{m} f(s_{i}) } - \frac{1}{m} \sum_{i=1}^{m} f(s_{i}) \right\rvert > t } \leq \exp\pars{ r - \frac{t^2 m }{ \tilde{c} (b-a)^2 } } ,
\end{align*}
where $\tilde{c}$ is an absolute constant. 
\end{theorem}
We note that the absolute constant $\tilde{c}$ in the Theorem above is estimated to be at most $0.5$ \cite{loffler2009shape}.

Using the result above, we prove the following bounds to the supports and qualities for all subgroups.
\begin{proposition}
\label{thm:unifboundsfreqqual}
Let $\lang$ be the language of subgroups composed by conjunctions with at most $z$ conditions over $d$ continuous features, and define $\mathsf{f}_{\patt} = \E_\dataset[\freqp]$.  
With probability $\geq 1 - \delta$ w.r.t. $\dataset$ it holds, for all $\patt \in \lang$, 
\begin{align*}
\freqp &\leq \mathsf{f}_{\patt} + \sqrt{ \frac{ 2z + \ln \bigl( \sum_{i=1}^{z} \binom{d}{i} \bigr) + \ln \bigl( \frac{2}{\delta} \bigr) }{ 2m } } , \\
\squalp &\geq \tqual - \sqrt{ \frac{ 2z + \ln \bigl( \sum_{i=1}^{z} \binom{d}{i} \bigr) + \ln \bigl( \frac{2}{\delta} \bigr) }{ 2m } } .
\end{align*}
\end{proposition}
\begin{proof}
Consider a dataset $\dataset$ with $d$ continuous features, and let $[1,d]$ be the indices of these features.
Let any $A \subseteq [1,d]$ with $|A| = v \leq z$, and define the language $\lang_A \subseteq \lang$ as all subgroups with $v$ conjunction terms involving conditions on the features of $A$, such as inequalities or intervals. 

We first prove the bound for $\freqp$. 
For any $\patt \in \lang_A$, recall the functions $f_\patt(s) = \ind{\patt \in s}$, and let $\F_A = \{ f_\patt : \patt \in \lang_A \}$ be the family of such functions. 
It is immediate to observe that the average value of $f_\patt$ over $\dataset$ is equal to $\freqp$. 
As discussed previously, $\F_A$ is equivalent to the class of axis-aligned rectangles in $\R^v$, with VC-dimension $2v$ (Problem~6.5 in \cite{ShalevSBD14}).
We apply \Cref{thm:pseudoappxthm} to the particular case of the binary function family $\F_A$, obtaining that 
\begin{align*}
\Pr \Bigl( \sup_{f_{\patt} \in \F_A} \left\lvert \mathsf{f}_{\patt} - \freqp \right\rvert > t \Bigr) \leq \exp ( 2v - 2 t^2 m ) .
\end{align*}
From an union bound over all subsets $A \subseteq [1,d]$ with cardinality $\leq z$, we have
\begin{align*}
&\Pr \Bigl( \sup_{A} \sup_{f_{\patt} \in \F_A} \left\lvert \mathsf{f}_{\patt} - \freqp \right\rvert > t \Bigr) \\
&\leq \sum_A \Pr \Bigl( \sup_{f_{\patt} \in \F_A} \left\lvert \mathsf{f}_{\patt} - \freqp \right\rvert > t \Bigr) 
\leq \sum_{i=1}^z \binom{d}{i} \exp ( 2i - 2 t^2 m ) .
\end{align*}
Setting the r.h.s. to be $\leq \delta/2$, we obtain the first bound. 

We now focus on the bound for $\squalp$.
For any $\patt \in \lang_A$, recall the functions $g_\patt(s,\ell) = \ind{\patt \in s}(\ell - \mu)$, and let $\mathcal{G}_A = \{ g_\patt : \patt \in \lang_A \}$ be the family of such functions. 
It is immediate to observe that the average value of $g_\patt$ over $\dataset$ is equal to $\squalp$. 
By using Lemma~3.6 of \cite{riondato2020misosoup}, 
we obtain that the pseudodimension of $\mathcal{G}_A$ is equal to the VC-dimension of $\F_A$, which is $2v$.
Therefore, we apply \Cref{thm:pseudoappxthm} to $\mathcal{G}_A$, having
\begin{align*}
\Pr \Bigl( \sup_{g_\patt \in \mathcal{G}_A} \left\lvert \tqual - \squalp \right\rvert > t \Bigr) \leq \exp ( 2v - 2 t^2 m ) .
\end{align*}
We complete the proof following the same steps for the frequencies. 
\end{proof}

We are now ready to prove \Cref{thm:power_unconditional}. 
\begin{proof}[Proof of \Cref{thm:power_unconditional}]
From the guarantees of \Cref{thm:corupperbounddevsub}, the value of $\varepsilon$ defined in the statement is an upper bound to the value of $\varepsilon$ computed by \algnameu\ in Algorithm~\ref{algo:main} with probability $\geq 1 - \delta/3$ (since $\lambda = \delta/3$ in the definition of $\hat{r}$).
We now prove a condition on the quality $\tqual$ for a pattern to be reported in output.
From \Cref{thm:unifboundsfreqqual} (replacing $\delta/2$ by $\delta/3$), 
a pattern with $\tqual \geq \xi$ has, with probability $\geq 1 - \delta/3$,
its estimate $\squalp \geq \xi - \varepsilon_1$, where 
\begin{align*}
\varepsilon_1 = \sqrt{ \frac{ 2z + \ln \bigl( \sum_{i=1}^{z} \binom{d}{i} \bigr) + \ln \bigl( \frac{3}{\delta} \bigr) }{ 2m } } .
\end{align*}
Moreover, from the relation between $\squalp$ and $\qualp$, it holds 
$\qualp \geq \xi - \varepsilon_1 - \varepsilon_T \freqp$.
Consequently, $\patt$ is reported in output if
$\qualp \geq \xi - \varepsilon_1 - \varepsilon_T \freqp \geq \varepsilon + \varepsilon_T \freqp$, so if
$\xi \geq \varepsilon_1 + 2\varepsilon_T \freqp + \varepsilon$.
Using the upper bound to $\freqp$ proved in \Cref{thm:unifboundsfreqqual} (replacing $\delta/2$ by $\delta/3$), the condition
$\tqual \geq \xi \geq \varepsilon_1 + 2\varepsilon_T (\tfreq + \varepsilon_1) + \varepsilon$
is sufficient to guarantee that, with probability $\geq 1 - \delta$, $\patt$ is reported in output, obtaining the statement. 
\end{proof}

\subsection{Baseline methods}
\label{sec:baselines}
In this Section we provide additional details on the baseline algorithm 
\algnameuub\ 
considered in our experimental evaluation. 
As anticipated in \Cref{sec:experiments}, it is not possible to apply standard statistical techniques (e.g., Bonferroni correction) to bound the FWER when the number $|\lang|$ of tested hypothesis is unbounded. 
Therefore, to overcome this issue we proceed as follows. 
Recall that $N_\lang(\dataset)$ is the number of distinct projections of the language $\lang$ on the dataset $\dataset$ (defined in \Cref{sec:app_power}). 
The key idea for \algnameuub\ is to apply a Bonferroni correction on the $N_\lang(\dataset)$ distinct tested hypothesis only. 
We do so with the following result, that defines the value of $\varepsilon$ used by \algnameuub\ in Algorithm~1 to bound the FWER. 
Note that \algnameuub\ does not use any resamples, but uses $N_\lang(\dataset)$ to obtain an analytical value for $\varepsilon$ instead. 
\begin{restatable}{theorem}{thmboundsupdevub}
\label{thm:bound_dev_ub}
Let $\dataset$ be a dataset of $m$ samples taken i.i.d. from a distribution $\probdist$. 
For any $\delta \in (0,1)$, define $\nu_T \geq \mu (1 - \mu)$, $\nu \geq \nu_T \sup_{\patt \in \lang^\star} \brpars{ \E_\dataset\sqpars{ \freqp } }$, $\hat{N} \geq N_\lang(\dataset)$, and  $\varepsilon$ as
\begin{align*}
& \hat{r} = \sqrt{ \frac{ \ln ( \frac{4 \hat{N} }{\delta} ) }{ 2m } } \\
& \hat{d} = \hat{r} + \sqrt{ \pars{\frac{2 \nu_T \unionbdelta }{ m }}^2 + \frac{ 2 \hat{r} \unionbdelta }{ m }} + \frac{2 \nu_T \unionbdelta }{ m } \\
& \varepsilon \doteq \hat{d} + \sqrt{\frac{2 \unionbdelta
    \left( \nu + 2\hat{d} \right)}{m}}
        + \frac{ \unionbdelta}{3m} . 
\end{align*}
With probability at least $ 1-\delta$ over $\dataset$ it holds
\begin{align*}
\sup_{\patt \in \lang^\star} \brpars{ \squalp  } \leq \varepsilon .
\end{align*}
\end{restatable}
\begin{proof}
We follow similar steps taken in the proof of \Cref{thm:bound_dev}; however, the key difference is in the upper bound $\hat{r}$ to $\E_{\resset}[ \devest \given \A ]$. 
We have, using Hoeffding's inequality and an union bound, that for any $t \geq 0$  
\begin{align*}
\Pr_{\resset}\pars{ \devest \geq t }
\leq N_\lang(\dataset) \exp\pars{ - 2 m t^2 }
\leq \hat{N} \exp\pars{ - 2 m t^2 } .
\end{align*}
This implies that, with probability $\geq 1-\delta/4$ over $\resset$, it holds
\begin{align*}
\devest \leq \sqrt{ \frac{ \ln ( \frac{4 \hat{N} }{\delta} ) }{ 2m } } .
\end{align*}
Taking the expectation w.r.t. $\resset$, conditionally on $\A$, proves that 
$\E_{\resset}[ \devest \given \A ] \leq \hat{r}$ and gives the statement. 
\end{proof}

\subsection{Details on CNN interpretation}
\label{sec:mnistexperiment}

In this section we provide all details regarding the application of \algname\ to interpret the activation of neurons in a CNN.
We considered the MNIST handwritten digit dataset, and train a CNN to predict the digit contained in each image.
Following the experimental setup of~\cite{fischer21b}, 
we employ a simple CNN composed by: 
two convolutional layers, with respectively $10$ and $40$ filters and $3\times 3$ kernels;
each convolutional layer is followed by $2 \times 2$ maxpooling and dropout of rate $0.25$;
a fully connected layer with $64$ nodes and \textit{ReLu} activations;
a dropout layer with rate $0.5$;
the output layer of size $10$ with softmax activations. 
We trained the network over the $60000$ images of the training set,
using SGD based on categorical cross entropy loss, with 
learning rate $0.01$, 
momentum $0.9$,
$12$ epochs, and batch size of $128$,
obtaining an accuracy of $0.987$ on the $10000$ holdout instances.
Then, analogously to~\cite{fischer21b}, we computed the activations 
of neurons in the first filter of the first convolutional layer on the $10000$ testing instances, obtaining $d = 676$ continuous features for $m=10000$ transactions. 
As binary target, we use the value $1$ for all digits containing straight lines only (the digits $1$ and $7$), and $0$ for all other digits.
Doing so, we obtain a fraction $\mu(\dataset)=0.216$ of samples with label $1$.
We ran \algnamec\ on this dataset, using $c=10$ resamples and $z=1$, to identify significant subgroups with FWER $\leq 0.05$.
For each neuron, we identify the most significant subgroup containing a condition over the corresponding continuous feature $x$.
We focus on conditions of the type $x \geq t$ and $x < t$, where $t$ is any real valued threshold; 
the condition $x \geq t$ denotes that the neuron $x$ is activated, with values at least $t$, while $x < t$ means that $x$ is inactive, with a value at most $t$.

In \Cref{fig:mnist} we show the results of this experiment. 
Figures~\ref{fig:mnist}-(a) and (b) show the average activation values of all neurons with label $1$ (corresponding to the digits $1$ and $7$) and label $0$ (all other digits). 
Then, in \Cref{fig:mnist}-(c) we show the neurons with activation conditions significantly associated to the target label $1$:
neurons that are significantly activated, i.e., with conditions $x \geq t$, are shown in red, while neurons that are significantly inactive, i.e., with conditions $x < t$, are shown in blue.
The color intensity (i.e., toward red or blue) depends on the threshold $t$: 
for activated neurons it saturates at the maximum activation value,
while for inactive neurons at the minimum. 
Interestingly, we observe that the neurons in this filter are significantly activated in the presence of a central vertical stroke (red pixels), that specifically characterizes the digits $1$ and $7$ (Figure~\ref{fig:mnist}-(a)). 
We remark, however, that this observation is not immediate from the average activations alone, and may be lost when the activation values are binarized, since \algname\ identifies subgroups that precisely separate the two classes of digits by using granular threshould values. 
On the countrary, neurons that are significantly inactive (blue pixels) are the ones surrounding this vertical area, forming two rounded shapes, which are typical of the other digits (Figure~\ref{fig:mnist}-(b)). 
Interestingly, these blue regions seem to shape the negative of the digits $1$ and $7$.

From this experiment we conclude that \algname\ can be effectively applied to identify complex patterns, such as the ones for the interpretation of the activations within a neural layer.

\begin{figure*}[ht]
\setlength{\fboxsep}{0pt}%
\setlength{\fboxrule}{0.8pt}%
\begin{subfigure}{.04\textwidth}
  \centering
  \includegraphics[width=\textwidth]{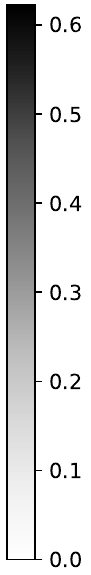}
  \vspace{0.2cm}
\end{subfigure}
\hspace{0.05cm}
\begin{subfigure}{.28\textwidth}
  \centering
  \fbox{\includegraphics[width=\textwidth]{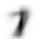}}
  \caption{}
\end{subfigure}
\hspace{0.5cm}
\begin{subfigure}{.28\textwidth}
  \centering
  \fbox{\includegraphics[width=\textwidth]{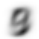}}
  \caption{}
\end{subfigure}
\hspace{0.5cm}
\begin{subfigure}{.28\textwidth}
  \centering
  \fbox{\includegraphics[width=\textwidth]{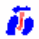}}
  \caption{}
\end{subfigure}
\caption{ 
Average activation of neurons of the first filter in the first convolutional layer, for the digits $1$ and $7$ (a) and for the other digits (b).
(c): neurons that are significantly activated for the digits $1$ and $7$ (in red) and inactive (in blue) identified by \algname. 
}
\label{fig:mnist}
\end{figure*}

\fi

\clearpage

\end{document}